\documentclass[runningheads]{llncs}
\usepackage{graphicx}

\usepackage{tikz}
\usepackage{comment}
\usepackage{amsmath,amssymb} 
\usepackage{color}

\usepackage{algorithm}  
\usepackage{algorithmicx} 
\usepackage{algpseudocode} 
\usepackage{amsthm}
\usepackage{diagbox}
\usepackage{cite}
\usepackage{float}
\usepackage{multirow}
\usepackage{enumitem}
\usepackage{setspace}
\usepackage{bbding}
\usepackage{booktabs}
\usepackage{graphicx}

\usepackage[breaklinks=true,
bookmarks=ture,citecolor=black,    
colorlinks=true,
linkcolor=black,
filecolor=magenta,      
urlcolor=cyan,]{hyperref}

\usepackage[capitalize]{cleveref}
\crefname{section}{Sec.}{Secs.}
\Crefname{section}{Section}{Sections}
\Crefname{table}{Table}{Tables}
\crefname{table}{Tab.}{Tabs.}

\begin{document}
\newtheorem{assumption}{Assumption}

\newtheorem*{theorem2}{Theorem}
\newtheorem*{lemma2}{Lemma}
\newtheorem*{corollary2}{Corollary}
\newtheorem*{assumption2}{Assumption}
\newtheorem*{definition2}{Definition}
\allowdisplaybreaks

\pagestyle{headings}
\mainmatter

\title{Balancing Stability and Plasticity through Advanced Null Space in Continual Learning} 

\titlerunning{Advanced Null Space}
%
\author{Yajing Kong\inst{1} \and
Liu Liu \inst{1} \and
Zhen Wang \inst{1} \and Dacheng Tao \inst{1, 2}}
\authorrunning{Y. Kong et al.}
%
\institute{The University of Sydney, Darlington, NSW 2008, Australia\\
\and JD Explore Academy, Beijing, China\\
\email{\{ykon9947, liuliu1, zwan4121\}@sydney.edu.au},   \email{dacheng.tao@gmail.com}}
\maketitle

\begin{abstract}
Continual learning is a learning paradigm that learns tasks sequentially with resources constraints, in which the key challenge is stability-plasticity dilemma, i.e., it is uneasy to simultaneously have the stability to prevent catastrophic forgetting of old tasks and the plasticity to learn new tasks well. In this paper, we propose a new continual learning approach, Advanced Null Space (AdNS), to balance the  stability and plasticity without storing any old data of previous tasks. Specifically, to obtain better stability, AdNS makes use of low-rank approximation to obtain a novel null space and projects the gradient onto the null space to prevent the interference on the past tasks. To control the generation of the null space, we introduce a non-uniform constraint strength to further reduce forgetting. Furthermore, we present a simple but effective method, intra-task distillation, to improve the performance of the current task. Finally, we theoretically find that null space plays a key role in plasticity and stability, respectively. Experimental results show that the proposed method can achieve better performance compared to state-of-the-art continual learning approaches. 
\keywords{Continual learning \and catastrophic forgetting \and 
null space}
\end{abstract}

\section{Introduction}
Humans have excellent abilities in learning new knowledge while maintaining the knowledge learned from past experience through their lifelong time.
Continual learning aims at developing algorithms for neural networks with the same capabilities from a stream of data \cite{parisi2019continual,tani2016exploring}.
However, although deep neural networks have made impressive achievements across various domains, they easily suffer performance degradation on the previous tasks when applied to sequential tasks without any access to historical data. The problem, referred to as catastrophic forgetting, is a key challenge in continual learning
  \cite{parisi2019continual,tani2016exploring,lee2017overcoming,mccloskey1989catastrophic,rebuffi2017icarl,kemker2018measuring}.

This problem is closely related to the stability-plasticity dilemma \cite{mirzadeh2020dropout,mirzadeh2020understanding}. Specifically, when learning in a sequential fashion, the network is required to have the plasticity to integrate new knowledge well and the stability to prevent the forgetting of previous tasks. However, the stability-plasticity dilemma indicates that it is hard to simultaneously have high plasticity and high stability. To relieve the dilemma, a growing body of continual learning methods are introduced. These method can be roughly divided into four categories: architecture-based methods expand the network or allocate new neurons for new tasks \cite{zhou2012online,rusu2016progressive,jerfel2019reconciling,mallya2018packnet}; replayed-based methods interleave old data with current data by storing historical data in a buffer or generating virtual old data \cite{rolnick2019experience,isele2018selective,buzzega2020dark,chaudhry2019continual,riemer2018learning};  regularization-based methods penalize the update of important parameters of previous tasks \cite{kirkpatrick2017overcoming,zenke2017continual,aljundi2018memory,lee2017overcoming}; algorithm-based methods modify the update rule of parameters to prevent the interference across tasks \cite{Wang_2021_CVPR,saha2021gradient,chaudhry2018efficient,lopez2017gradient,tang2021layerwise}. 

For algorithm-based methods, one of the classical approaches is to project the gradient onto the approximation null space of all previous tasks, in which the gradient has little interference on the performance of previous tasks \cite{Wang_2021_CVPR,saha2021gradient,zeng2019continual}. 
However, despite the impressive performance achieved by these methods, there are still some challenges that impede the null space methods to achieve satisfactory stability-plasticity trade-off. 
First, the null space methods are based on the finding that the model modifies the parameters in the exact null space of previous tasks.
However, due to the approximation of null space, the model will occur in interference on the previous tasks. 
Moreover, the interference would affect the subsequent approximation of null space of past tasks, thus leading to more information deficiency of null space of previous tasks. 
Therefore, the stability of the model will be unsatisfactory.
Second, the model update is based on the gradient projection on the null space of previous task, preventing the model from learning the current task well,
i.e., resulting in worse plasticity. 
    \begin{figure*}[t]
\centering
   \includegraphics[width=1.0\linewidth]{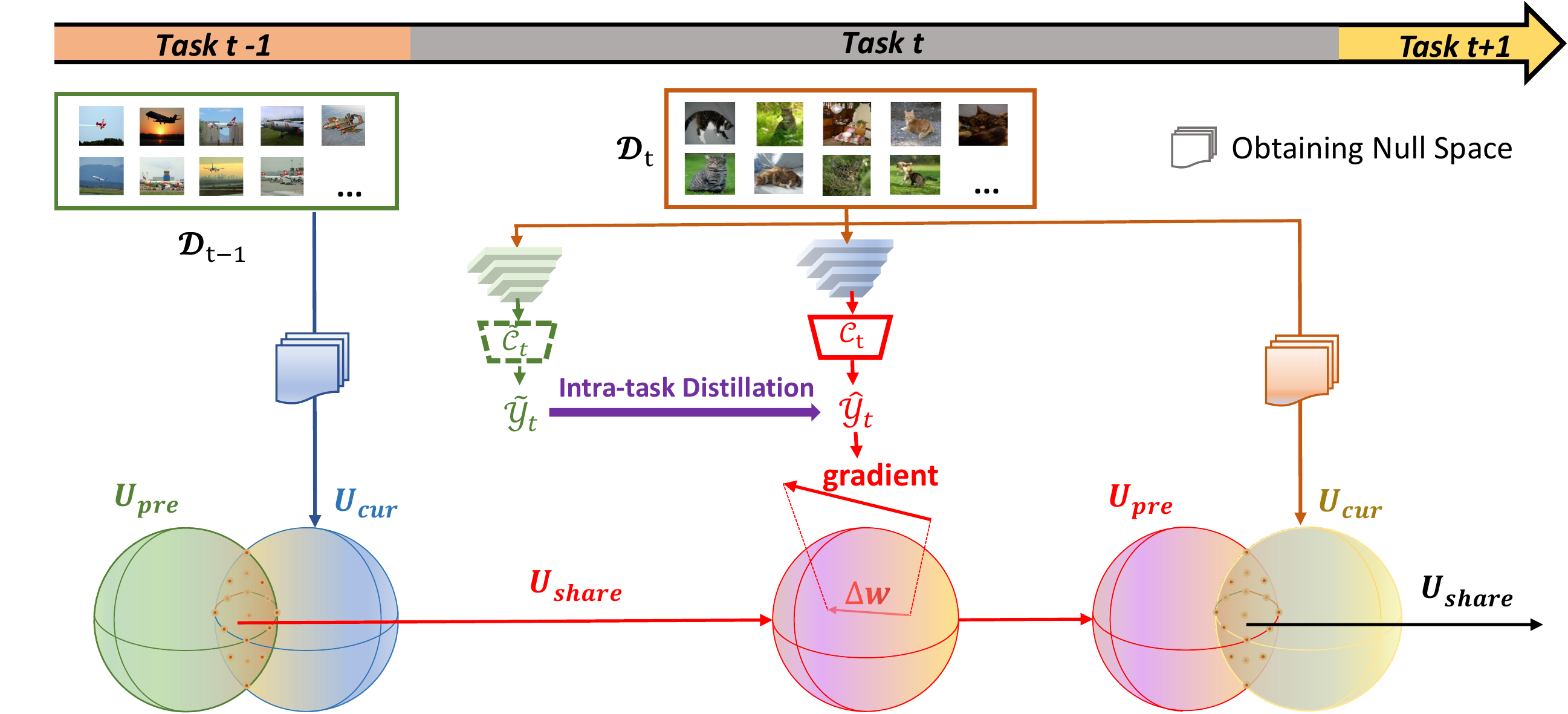}
   \caption{The pipeline of the proposed method. Left: At the task $\mathcal{T}_{t-1}$, we obtain the shared low-rank null space  $\mathbf{U}_{\text {share}}$ based on $\mathbf{U}_{\text {pre}}$  and $\mathbf{U}_{\text {cur}}$. Middle: We project the gradient at each layer onto the shared low-rank null space. Right: $\mathbf{U}_{\text {share}}$ is used for the next task $\mathcal{T}_{t+1}$ as $\mathbf{U}_{\text{pre}}$.
   Note at the task $\mathcal{T}_{t}$ $(t\textgreater 1)$, we conduct Intra-task Distillation between $\tilde{\mathcal{Y}}_t$ and $\hat{\mathcal{Y}}_t$. 
   }
\label{LowRank:procedure}
\end{figure*}

 To address the above challenges,
  we propose a new algorithm-based continual learning approach, Advanced Null Space (AdNS), to achieve a good balance between stability and plasticity. Specifically, to alleviate the impact of approximation on the stability, 
  AdNS makes use of the low-rank approximation to extract the shared null space between the previous null space and the current candidate null space. Unlike existing works that only focus on the current candidate null space \cite{saha2021gradient, zeng2019continual}, AdNS projects the gradient onto the shared null space,  which contains the core spaces between null spaces, and thus could mitigate the information deficiency of the previous null spaces and reduce forgetting.
   Moreover, we present a constraint to control the approximation of null space and propose non-uniform constraint strength, which 
monotonically decreases with the number of tasks increasing, to further relieve the forgetting.     
 What's more, to improve the performance of  the current task, we leverage a simple method, intra-task distillation, to self-distill the knowledge of the current task.
 The procedure of the proposed method is shown in Fig. \ref{LowRank:procedure}.

  Finally, although various algorithms about null space have been proposed, few efforts were spent on the theoretical foundations. Therefore, in this paper, we theoretically analyze the impact of  null space and present two theorems to prove that  null space plays a key role in stability and plasticity.  
 The theoretical finding indicates the inherent properties of the stability-plasticity dilemma, in which it is hard to have high plasticity and high stability simultaneously.
 To summarize, our contributions are threefolds:
\begin{itemize}
[topsep=0pt,itemsep=-1ex,partopsep=1ex,parsep=1ex, leftmargin=0.4cm]
    \item 
    To address the stability-plasticity dilemma, we propose a new algorithm-based continual learning approach, AdNS, which projects the gradient into the shared null space under non-uniform constraint strength to reduce forgetting, and uses intra-task distillation to improve the tasks' performance.
    \item We present two theorems from the perspective of stability and plasticity, which show that the null space plays a key role in balancing the stability-plasticity dilemma.
    Specifically, the larger the dimension of the null space, the better the plasticity, the worse the stability.
    \item 
    We validate the proposed method in several benchmarks, and the empirical results show that the proposed method can outperform related state-of-the-art continual learning methods.
\end{itemize}

\section{Related Work}

\noindent\textbf{Algorithms-based methods} design the update rule to decrease the interference of parameter update on the performance of old tasks \cite{lopez2017gradient, chaudhry2018efficient, Wang_2021_CVPR, saha2021gradient, tang2021layerwise, chaudhry2020continual, trgp, aop}. For example,
GEM \cite{lopez2017gradient} and A-GEM \cite{chaudhry2018efficient} used the historical samples to compute the gradients of old tasks and proposed inequality constraints of gradients to avoid the increase of losses of past tasks. \textit{Arslan et al.} \cite{chaudhry2020continual} manually divided a random orthonormal space into several subspaces and allocated these subspaces one-to-one to each task. However, these methods require storing data of previous tasks. 
In contrast,
GPM \cite{saha2021gradient} stored the bases of core gradient space and modified the parameters
in the direction orthogonal to the core space. 
OWM \cite{zeng2019continual} modified the parameters in the direction orthogonal to the input space of previous tasks. 
 Adam-NSCL stored uncentered feature covariance and used it to compute the null space \cite{Wang_2021_CVPR}. 
Our work is closely related to Adam-NSCL.
However, 
unlike Adam-NSCL that only considers the current candidate null space,
we project the gradient onto the shared null space between null spaces,
which could relieve the information deficiency of null space of previous tasks,
resulting in less forgetting.
Moreover, 
our method does not rely on any data of previous tasks and only needs to update the null space in one shot manner at the end of each task.

\noindent\textbf{Regularization-based methods} can be divided into two classes: one is to distill knowledge from the previous model which is trained on the previous tasks \cite{li2017learning, jung2016less, zhang2020class, rannen2017encoder,jing2021amalgamating, ICML20_DSL}; another is to explicitly use a regularizer to penalize the update of important parameters of previous tasks, preventing the model from deviating too much from the previous one
to avoid forgetting
\cite{lee2020continual, lee2017overcoming, park2019continual, nguyen2017variational,zenke2017continual}.
For example, for the first class, LwF \cite{li2017learning} distills the knowledge by using the previous model outputs as soft labels
and penalizing the distillation term between the current and recorded output. 
For the latter one, EWC \cite{kirkpatrick2017overcoming} used the diagonal of the Fisher information matrix as the importance.

\noindent \textbf{Other Approaches.} Architecture-based methods allocate different parameters or add new parameters for the new task, while sharing parameters across tasks to reduce the interference of previous tasks \cite{aljundi2017expert,li2019learn, zhou2012online, rusu2016progressive,wu2020firefly,yoon2019scalable,rosenfeld2018incremental, serra2018overcoming, mallya2018packnet, mallya2018piggyback,masse2018alleviating}.
However, such methods may lead to a cumbersome and complex network if new tasks continually arrive. 
Replayed-based methods leverage episodic memory to store representative history data or generate virtual data via a generative model, and replay these samples with current data \cite{rolnick2019experience, isele2018selective,  chaudhry2019continual, riemer2018learning,lopez2017gradient, chaudhry2018efficient, shin2017continual, rao2019continual, aljundi2019online, CVPR22_LVT, AAAI22_CL}. 
However, replayed-based methods may bring some problems since storing old data will result in data imbalance, and the generative model would be large and expensive
if it synthesizes the historical data reasonably.

\section{Preliminaries}
\subsection{Settings and Notations}
In this part, we present the settings and notations. We consider  a sequence of tasks $\mathcal{T}_{t}$, $t \in \{1, ..., T\}$, where $T$ is the total number of tasks.
Let $\mathcal{D}_{t} = \{\mathcal{X}_{t}, \mathcal{Y}_{t}\}$
be the dataset of task $\mathcal{T}_{t}$, where $\mathcal{X}_{t}$ and $\mathcal{Y}_{t}$ are the corresponding inputs set and label set.
In continual learning, the model is trained on these datasets sequentially.
Let $\mathbf{w} = \{w^1,..., w^{L}\}$ be the parameters of $L$-layer neural network, where $w^l$ is the parameter vector of the $l$-th layer, $l \in \{1,..., L\}$. 
\textcolor{black}{Let $\hat L_{t}(\mathbf{w})$ be the empirical loss of task $\mathcal{T}_t$ with parameters $\mathbf{w}$.} Define $\tilde{\mathbf{w}}_t$ as the convergence parameters after the model has been trained on the task $\mathcal{T}_t$. Rank($\cdot$) denotes the rank of a matrix, and $[\cdot, \cdot]$ denotes the concatenation of vectors or matrices. $\|\cdot\|_{F}$ denotes Frobenius norm,   $\|\cdot\|_{1}$ denotes L$_1$ norm, and $\|\cdot\|_{2}$ denotes L$_2$ norm.

\subsection{Null Space}

Let $\Delta w^l$  be the parameter update for $l$-th layer for the current step.
If $\Delta w^l$ lies in the null space of previous tasks at each training step for the $l$-th layer, $l \in \{1, ..., L\}$,  then the stability can be guaranteed, which is illustrated by the following lemma. 

\begin{lemma}
 \cite{Wang_2021_CVPR} Define $X_{p,p}^l$ as the input feature of $l$-th layer when the network is fed with data $\mathcal{X}_p$ after training on the task $\mathcal{T}_p$. Let $\mathcal{N}(\tilde{\mathbf{w}}_t; \mathcal{X}_p)$ be the output of the $L$-layer network with parameters $\tilde{\mathbf{w}}_t$ when the network is fed with data $\mathcal{X}_p$.
If at each training
step of task $\mathcal{T}_t$, $\Delta w^l$ lies in the null space of $X_{t-1}^l = [X_{1,1}^l, ..., X_{t-1,t-1}^l]$ , i.e.,
\begin{align}
\label{LowRank:null}
    X_{t-1}^l\Delta w^l = \mathbf{0},  \quad  l = 1, ..., L,
\end{align}
 then we have $\mathcal{N}(\tilde{\mathbf{w}}_t; \mathcal{X}_p) = \mathcal{N}(\tilde{\mathbf{w}}_p; \mathcal{X}_p )$ for all $p \in \{1,..., t -1 \}$.
 \label{LowRank:Lemma1}
\end{lemma}

According to Lemma \ref{LowRank:Lemma1}, if the parameters are modified in the null space of $X_{t-1}^l$, 
then the training loss of previous tasks will be retained and the forgetting can be avoid.
Nevertheless, it is unrealistic to expect the existence of the null space, thus previous works \cite{Wang_2021_CVPR, zeng2019continual, saha2021gradient} use the approximation null space instead. However, the approximation will cause that Eq.(\ref{LowRank:null}) no longer holds and result in the occurrence of interference on previous tasks. Moreover, the interference would affect the subsequent approximation of null space of past tasks, leading to more performance degradation on past tasks, i.e., catastrophic forgetting.

\section{Methodology}
In this section, we propose a new continual learning method, Advanced Null Space (AdNS), involving shared low-rank null space (Section \ref{LowRank:sectionLowRank}), non-uniform constraint strength (Section \ref{LowRank:Strength}), and intra-task distillation (Section \ref{LowRank:SelfDistillation}), to balance the stability and plasticity.  The procedure of AdNS is shown in Fig. \ref{LowRank:procedure} and the algorithm is shown in Algorithm \ref{LowRank:Algorithm}.  
\floatname{algorithm}{Algorithm}  
\renewcommand{\algorithmicrequire}{\textbf{Input:}}  
\renewcommand{\algorithmicensure}{\textbf{Output:}} 
\begin{algorithm}[t]
\caption{Advanced Null Space (AdNS)}  
\begin{algorithmic}
\Require Network $\mathcal{N}$ with parameters $\mathbf{w}$
\Ensure Target network $\mathcal{N}$
\For{$t=1,2,..,T$}
\If{$t = 1$}
\While{not converged}
\State  Update the gradient according to  $ \hat L_1(\mathbf{w})$
\EndWhile
\State $\tilde X_{1}^l \leftarrow  (X_{1,1}^l)^{\top}X_{1,1}^l$
\State $\mathbf{U}^{l} \leftarrow $  Null space of $\tilde X_{1}^l$ based on 
(\ref{LowRank:constraint})
\State $\mathbf{U}_{\text{pre}}^{l} \leftarrow \mathbf{U}^{l}$ 
\State Break \hfill $\triangleright$ Return to the next task
\EndIf
\State $\tilde{\mathcal{Y}}_{t} \leftarrow \mathcal{N}(\mathbf{w};\mathcal{X}_t)$ after updating the classifier  $\tilde{\mathcal{C}}_t$
\While{not converged}
\State Update the gradient in the null space  at each layer based on  $\mathcal{U} = \{\mathbf{U}^1, ..., \mathbf{U}^L\}$ according to ~\eqref{LowRank:OptimProblem}
\EndWhile
\State $\tilde X_{t}^l \leftarrow \tilde X_{t-1}^l + (X_{t,t}^l)^{\top}X_{t,t}^l$
\State $\mathbf{U}_{\text{cur}}^{l} \leftarrow $  Null space of $\tilde X_{t}^l$ based on 
(\ref{LowRank:constraint})
\State $\mathbf{\tilde U}^l = [\mathbf{U}_{\text{pre}}^{l}, \mathbf{U}_{\text{cur}}^{l}]$ $,l = 1,..., L$
\State $\mathbf{U}^{l} \leftarrow$ The shared low-rank null space obtained in (\text{P}\ref{LowRank:OptimizationProblem})
\State $\mathbf{U}_{\text{pre}}^{l} \leftarrow \mathbf{U}^{l}$ 
\EndFor
\State \Return $\mathcal{N}$
\end{algorithmic}
\label{LowRank:Algorithm}
\end{algorithm}
\subsection{Shared Low-Rank Null Space}
\label{LowRank:sectionLowRank}
In this part, we introduce a noval null space,
the shared low-rank null space, which extracts the shared null spaces between the previous null space and the current candidate null space based on low-rank approximation.

When training on the current task $\mathcal{T}_t$ ($t \textgreater 1$), the gradient of $l$-th layer is projected onto the null space of previous tasks, which is spanned by the columns of $\mathbf{U}_{\text{pre}}^l$\footnote{We use the matrix whose columns are consisted of the orthonormal basis of the null space to represent null space.},
$l\in \{1,...,L\}$ and 
$ \mathcal{U}_{\text{pre}} = \{\mathbf{U}_{\text{pre}}^1, ..., \mathbf{U}_{\text{pre}}^L\}$  is the set containing the previous null spaces of $L$ layers.
After training the current task, the current candidate null space $\mathcal{U}_{\text{cur}} = \{\mathbf{U}_{\text{cur}}^1, ..., \mathbf{U}_{\text{cur}}^L\}$ could be obtained based on the input features of tasks seen so far. Specifically, we use the uncentered feature covariance, i.e., $\tilde X_{t}^l = (X_{t}^l)^{\top}X_{t}^l$, to compute the current candidate null space \cite{Wang_2021_CVPR}. Such process has moderate memory consumption since the dimension of $\tilde X_{t}^l$ is irrelevant to the size of data. It can be easily proved that the null space of $\tilde X_{t}^l$ is equal to $X_{t}^l$. We update the input features of each layer by $\tilde X_{t}^l \leftarrow \tilde X_{t-1}^l + (X_{t,t}^l)^{\top}X_{t,t}^l$, where $X_{t,t}^l$ is the input feature of $l$-th layer when the network is fed with data $\mathcal{X}_t$ after training on the task $\mathcal{T}_t$, and then obtain the current candidate null space at the end of task.

To alleviate the impact of approximation on the stability, 
rather than using $\mathcal{U}_{\text{cur}}$, we extract the shared null space 
between $\mathcal{U}_{\text{pre}}$ and $\mathcal{U}_{\text{cur}}$ by solving the problem of low-rank approximation for the concatenation matrix 
$\mathbf{\tilde U}^l = [\mathbf{U}_{\text{pre}}^l,   \mathbf{U}_{\text{cur}}^l]$:
\begin{equation}
    \begin{aligned}
     \text{minimize}_{\mathbf{ \hat U}^l}   \quad \| \mathbf{\tilde U}^l -  \mathbf{ \hat U}^l\|_F \nonumber \quad
     s.t. \quad \text{Rank}(\mathbf{ \hat U}^l) \leq k_{l}, \quad l = 1,...,L,
    \label{LowRank:OptimizationProblem}
    \end{aligned}
    \eqno{(\text{P}2)}
\end{equation}
where $k_{l}$ is the rank of the shared null space of $l$-th layer.
The optimization problem (P\ref{LowRank:OptimizationProblem}) has analytic solutions in terms of the singular value decomposition.
 Because we want to project the gradient onto an orthonormal space, according to the properties of singular value decomposition, the objective matrix $\mathbf{U}^l$, i.e., the shared low-rank null space of $l$-th layer,
can be
constructed by the singular vectors of $\mathbf{\hat U}^l$ rather than using $\mathbf{\hat U}^l$ directly. The implementation details can be found in the Appendix.

The shared low-rank null space, which is the range space of $\mathbf{U}^l$, contains the shared information of the previous null space and the current null space. 
Thus, projecting the gradient onto the shared low-rank space could relieve catastrophic forgetting.
In formal,
we project $g^l$ as the following projection operation:
\begin{align}
    \Delta w^{l} = \mathbf{U}^l(\mathbf{U}^l)^{\top}g^l, \quad l = 1,...,L,
\label{LowRank:projection}
\end{align}
and the parameter is updated by $w^{l} \leftarrow w^{l} - \eta \Delta w^{l}$, where $\eta$ is the learning rate.

Regarding the computing complexity, let the dimension of the feature at $l$-th layer be  $d^l$ and the dimension of $\tilde{ \mathbf{U}}^l$ be $\tilde{k}_l$, where $\tilde{k}_l \textless d^l$.
Then the complexity of computing the current candidate null space $\mathcal{O}((d^l)^3)$ is larger than the complexity  of the low-rank approximation $\mathcal{O}(d^l(\tilde{k}_l)^2)$. 
Therefore, the time consumption of our proposed method is comparable to previous works \cite{Wang_2021_CVPR, saha2021gradient}. The comparisons of running time can be reffed to Appendix.

\subsection{Non-uniform Constraint Strength}
\label{LowRank:Strength}
In continual learning, it is essential to balance the importance of previous tasks and the current task to achieve satisfying performance \cite{lee2020continual, lee2017overcoming, zenke2017continual}. For example, if the model puts too much weight on the current task, it may suffer significant performance degradation on the prior tasks and  vice-versa.
Therefore, to make a trade-off between the importance of previous tasks and current task, i.e., stability and plasticity, we rewrite the constraint $X_{t-1}^l\Delta w^l = \mathbf{0}$ for $l \in \{1,...,L\}$ to $\|X_{t-1}^l\Delta w^l\|_1 \leq \epsilon$, where $\epsilon$ is a factor controlling the strength of constraint. 
Such a constraint is reasonable for the following two reasons: (a)
In practice, it is unrealistic to expect that there exists a null space satisfying 
$X_{t-1}^l\Delta w^l = \mathbf{0}$ for $l \in \{1,...,L\}$. (b) 
Although the constraint $X_{t-1}^l\Delta w^l = \mathbf{0}$ could guarantee that the model would not suffer from catastrophic forgetting, it would result in poor performance of the current task because the constraint is too strict for the parameter update of the current task. 
In contrast, introducing the balance factor $\epsilon$ allows us to make a trade-off between stability and plasticity flexibly.

Moreover, with new tasks occurring, the number of previous tasks increases, resulting in a greater impact of previous tasks on the final performance than the current task \cite{lee2020continual, lee2017overcoming, zenke2017continual}.  Hence, with the growth of observed tasks, it is necessary to pay more attention to previous tasks to relieve the catastrophic forgetting. Therefore,
 we propose non-uniform constraint strength following a common assumption that the importance of previous tasks is related to the number of tasks seen so far.
In particular, the non-uniform constraint strength can be represented as
\begin{align}
   \| X_{t-1}^l\Delta w^l \|_1 \leq \epsilon(t), \quad \text{for} \quad l = 1,...,L,
    \label{LowRank:constraint}
\end{align}
where $\epsilon(t)$ is a function monotonically decreasing with the number of tasks seen so far. With new tasks continually coming, the constraint strength becomes more restrictive, and thus the model pays more attention to preserving the performance of previous tasks, achieving better stability. 

\subsection{Intra-task Distillation}
\label{LowRank:SelfDistillation}
To further address the plasticity-stability dilemma, we leverage knowledge distillation of the current task, called \text{intra-task distillation}, to improve the performance of the current task. 
In particular, as shown in Fig. \ref{LowRank:procedure},
before training the task $\mathcal{T}_t$, we first freeze the backbone, which is learned on the previous $t-1$ tasks sequentially, and only train the classifier $\tilde{\mathcal{C}}_t$ of the task $\mathcal{T}_t$. Then we store the outputs $\tilde{\mathcal{Y}}_t$ for the current task $\mathcal{T}_t$. The frozen backbone and the classifier absorb the information of prior tasks and the current task, respectively.
Therefore, by penalizing the difference between the record outputs $\tilde{\mathcal{Y}}_t$ from the classifier $\tilde{\mathcal{C}}_t$ and the current outputs $\hat{\mathcal{Y}}_t$ from the classifier $\mathcal{C}_t$, intra-task distillation could improve the performance of the current task while preserving the acquired knowledge from previous tasks.

Specifically, we use the modified cross-entropy loss as the distillation loss. When training on the task $\mathcal{T}_t$, the distillation loss can be represented as:
\begin{equation}
    \begin{aligned}
    \hat L_{d}(\tilde{y}_t, \hat y_t) = -\sum_{c=1}^{C_t} \tilde{y}_t'^{(c)}\text{log}\hat y_t'^{(c)}, 
\end{aligned}
\label{lowrank:DistillationLoss}
\end{equation}
where $\tilde{y}_t'^{(c)} = \frac{\text{exp}(\tilde{y}_t^{(c)} /\tau)}{\sum_i \text{exp}(\tilde{y}_t^{(i)} /\tau)}, \hat y_t'^{(c)} = \frac{\text{exp}(\hat y_t^{(c)} /\tau)}{\sum_i \text{exp}(\hat y_t^{(i)}/\tau)}$, $C_t$ is the number of classes of task $\mathcal{T}_t$, $\tau$ is the temperature factor; $\tilde{y}_t^{(c)}$ and $\hat y_t^{(c)}$ are the recorded and current outputs of a sample $x_t$ in task $\mathcal{T}_t$, respectively. We set $\tau = 2$ by default.

In summary, when training on the task $\mathcal{T}_t$ $(t \textgreater 1)$, the optimization problem including (P\ref{LowRank:OptimizationProblem}) and (\ref{LowRank:projection})-(\ref{lowrank:DistillationLoss}), which is represented as:
\begin{equation}
    \begin{aligned}
    \mathop{\text{minimize}} \limits_{\mathbf{w}}  & \quad \hat L_t(\mathbf{w})  + \beta \hat L_d(\tilde{\mathcal{Y}}_t, \hat{\mathcal{Y}}_t),   \\
      s.t. &\quad   \text{Rank}( h(X_{t-1}^l)) \leq k_{l}, 
     \|X_{t-1}^l\Delta w^l\|_1 \leq \epsilon(t), \quad l = 1, ..., L,
\label{LowRank:OptimProblem}
    \end{aligned}
\end{equation}
where $\hat L_t(\mathbf{w})$ is the cross-entropy loss for the current task and $\beta$ is a coefficient that balances the importance between the cross-entropy loss and the distillation loss $\hat L_d(\tilde{\mathcal{Y}}_t, \hat{\mathcal{Y}}_t)$; $\mathbf{w}$ is the network parameters and $\epsilon(t)$ is the constraint strength which impacts the trade-off between the stability and plasticity;
$h(X)$ denotes $X$ is mapped to its approximation null space which satisfies the constraint;
$k_l$ is used to control the rank of the matrix of $l$-th layer.

\section{Analysis}
\label{LowRank:Analysis}
In this section, we present two theorems to theoretically prove that the null space plays a key role in the  stability-plasticity dilemma. Theorem \ref{Null_space:convergence_rate} is in terms of plasticity and Theorem \ref{Null_space:forgetting} is in terms of stability. The proof can be found in Appendix.
Before introducing the two theorems, we first present three assumptions. 
Comparable assumptions are also made in existing studies on continual learning \cite{yin2020optimization}.
\begin{assumption}
\label{LowRank:Assumption:sigmoid-distance}
     $\hat L_t(\mathbf{w})$ is $L_f$-smooth, i.e.,
        $\hat L_t(\mathbf{w})\leq \hat L_t(\mathbf{v})+\langle \nabla \hat L_t(\mathbf{v}),\mathbf{w}-\mathbf{v}\rangle+\frac{L_f}{2}\|\mathbf{w}-\mathbf{v}\|_2^2$, $t\in \{1,...,T\}$, for any $\mathbf{v},\mathbf{w}\in \mathbb{R}^d$.
\end{assumption} 
\begin{assumption}
\label{LowRank:Assumption:sameT}
     For each task $\mathcal{T}_t$, $t\in \{1,...,T\}$, the number of iterations is bounded by an integer $S$. 
\end{assumption}
\begin{assumption}
\label{LowRank:Assumption:sigma}
  $\hat L_t$ has the $\sigma^2$-uniformly bounded gradient variance, i.e., 
    \begin{align*}
        \| \nabla \hat L_t(\mathbf{w};x,y)-\nabla\hat L_t(\mathbf{w}; \mathcal{X}_{t}, \mathcal{Y}_{t})\|_2^2\leq \sigma^2, (x,y)\in \mathcal{D}_t, 
        t\in \{1,...,T\}.
    \end{align*}
\end{assumption}
Based on the assumptions, we derive the following two Theorems regarding the plasticity and stability. Specifically, we obtain the upper bound of the loss of the current task and forgetting.
\begin{theorem} (Plasticity)
\label{Null_space:convergence_rate}
Suppose Assumptions \ref{LowRank:Assumption:sigmoid-distance}, \ref{LowRank:Assumption:sameT}, and  \ref{LowRank:Assumption:sigma} hold. Let $\mathbf{w}_{t, s}$ be the parameters on task $\mathcal{T}_t$ at the $s$-th step and $\eta$ be the learning rate. Let the range of space of $\mathbf{U}^l$ be the null space of previous tasks for $l$-th layer,
then the loss of the current task $\mathcal{T}_{t}$ is upper bound by
{\small
\begin{align*}
    \hat L_t(\mathbf{w}_{t,S}) 
    \!\leq\! \hat L_t(\mathbf{w}_{t, 0})
    \!+\!\frac{\eta}{2}\sum_{s=0}^{S\!-\!1}\sum_{l = 1}^{L}{\|(I\!-\!
    \mathbf{U}^l(\mathbf{U}^l)^{\top})g_{t, s}^{l}\|_2^2} 
    \!-\!\frac{\eta}{2}\sum_{s=0}^{S\!-\!1}\|\nabla \hat L_t(\mathbf{w}_{t,s})\|_2^2 
    \!+\! \frac{S L_f \eta^2\sigma^2}{2},
\end{align*}
}
where $g_{t, s}^{l}$ is $l$-th layer gradient of  $ \hat L_t(\mathbf{w}_{t, s})$.
\end{theorem}
\begin{theorem} (Stability)
\label{Null_space:forgetting}
Suppose Assumptions \ref{LowRank:Assumption:sigmoid-distance}, \ref{LowRank:Assumption:sameT}, and  \ref{LowRank:Assumption:sigma} hold. Let $\mathbf{w}_{t, s}$ be the parameters on  task $\mathcal{T}_t$ at the $s$-th and $\eta$ be the learning rate. Let $\hat L_{1:t-1}$ be the sum of empirical loss function of previous $t-1$ tasks and $g_{1:t-1, s}^{l}$ is its gradient of $l$-th layer at $\mathbf{w}_{t, s}$. 
Let
$g_{t, s}^{l}$ be the gradient of the current task at $\mathbf{w}_{t, s}$ of $l$-th layer. 
Let the range of space of $\mathbf{U}^l$ be the null space of previous tasks for $l$-th layer, then the forgetting of the previous $t-1$ tasks generated by training on the task $\mathcal{T}_t$ is upper bound by 
\begin{align*}
    \hat L_{1:t-1}(\mathbf{w}_{t, S})   -  \hat L_{1:t-1}(\mathbf{w}_{t, 0})
    \leq &  
    \eta \sum_{s=0}^{S-1}\sum_{l=1}^{L} \|\mathbf{U}^l(\mathbf{U}^l)^{\top}\|_2 \|g_{t, s}^l\|_2 \| g_{1:t-1, s}^{l}\|_2 
    \\
    &+ \frac{L_f}{2}\eta^2\sum_{s=0}^{S-1}\sum_{l=1}^{L}\|\mathbf{U}^l(\mathbf{U}^l)^{\top}\|_2^2\|g_{t, s}^l\|_2^2.
\end{align*}
\end{theorem}
\begin{remark}
From Theorems \ref{Null_space:convergence_rate} and \ref{Null_space:forgetting} , we could conclude that $ \mathbf{U}^l$ plays a key role in the stability and plasticity. 
According to  Theorem \ref{Null_space:convergence_rate}, if the rank of the null space is larger, 
then the term $\|(I-  \mathbf{U}^l(\mathbf{U}^l)^{\top})g_{t, s}^{l}\|_2^2$ will be smaller and the upper bound of $\hat L_t(\mathbf{w}_{t,S})$ will be smaller. Therefore, the model could learn the current task better, indicating better plasticity. However, according to  Theorem \ref{Null_space:forgetting}, the larger the rank of the null space, the larger the term $\|\mathbf{U}^l(\mathbf{U}^l)^{\top}\|_2^2$. Therefore, the upper bound of the forgetting would be larger, resulting in poorer stability.
The two theorems indicate the inherent properties of the stability-plasticity dilemma that it is hard to have high plasticity and high stability simultaneously.  
\end{remark}

\section{Experiments}
\subsection{Experimental Setup}
\noindent \textbf{Datasets and Architecture.} Following \cite{Wang_2021_CVPR}, we perform experiments on three continual learning benchmarks: 10-Split CIFAR-100 (10-S-CIFAR100), 20-Split-CIFAR-100 (20-S-CIFAR100), and 25-Spilt TinyImageNet (25-S-TinyImageNet). Specifically, 10-Split CIFAR-100 and 20-Split CIFAR-100 are constructed by splitting CIFAR100 \cite{krizhevsky2009learning} into 10 and 20 sequential tasks, respectively. Each task contains the same classes without replacement out of the total 100 classes. Similarly, 25-Spilt TinyImageNet
is constructed by splitting 200 classes of TinyImageNet \cite{tinyimgnet} into 25 sequential tasks, where each task has 8 classes.
We use ResNet-18 \cite{he2016identity} as the backbone \cite{he2016identity, buzzega2020dark, guo2020improved, chaudhry2020continual}. 
All tasks share the same backbone, while each task has its separate classifier.

\noindent \textbf{Baselines.}
We compare the proposed method against competitive and well-established methods\footnote{We do not compare with replay-based methods because they store the data of previous tasks, which is out of the scope of this paper's setting.}, including 5 regularization-based methods using importance measure (EWC \cite{kirkpatrick2017overcoming},
MAS\cite{aljundi2018memory}, MUC-MAS\cite{liu2020more}, SI \cite{zenke2017continual}, and CPR \cite{cha2020cpr}), 2 regularization-based methods using knowledge distillation (LwF\cite{li2017learning} and GD-WILD 
\cite{lee2019overcoming}), 1 architecture-based method (InstAParam\cite{chen2020mitigating}), and 6 algorithm-based methods (GEM\cite{lopez2017gradient}, A-GEM\cite{chaudhry2018efficient}, MEGA\cite{guo2020improved}, 
OWM\cite{zeng2019continual}, GPM\cite{saha2021gradient}, and Adam-NSCL\cite{Wang_2021_CVPR}). We also provide a lower bound performance of Vanilla which trains tasks sequentially without any countermeasure to forgetting.

\noindent \textbf{Performance Metrics.}
To evaluate the performance,
we use two standard metrics: 
 a)
Average accuracy (ACC)~\cite{mirzadeh2020understanding, lopez2017gradient}
is the average test accuracy evaluated on all tasks after learning all tasks sequentially;
b) Backward Transfer (BWT)~\cite{lopez2017gradient, chaudhry2018efficient} is the average performance decrease of the network on previous tasks after new learning. 
In formal, ACC and BWT are defined as:
$    \text{ACC} = \frac{1}{T}\sum_{i=1}^{T}A_{T, i} $, 
$    \text{BWT} = \frac{1}{T-1}\sum_{i=1}^{T-1}(A_{T, i} - A_{i, i})$,
where $T$ is the total number of tasks and $A_{j, i}$ is the accuracy of task $\mathcal{T}_i$ after training on the task $\mathcal{T}_j$ sequentially. The larger the two metrics, the better the model. If the performances of ACC are similar, then the method with a larger value of BWT is better \cite{lopez2017gradient}.

\noindent \textbf{Implementation Details.}
\label{LowRank:Implementation_details}
  When obtaining the null space at the end of each task, we approximate the constraint $\|X_{t-1}^l\Delta w^l\|_1 \leq \epsilon(t)$ like \cite{Wang_2021_CVPR}.
  Specifically, when computing the current candidate null space, we approximate the current candidate null space with the singular values satisfying $\lambda \in \{\tilde \lambda | \tilde \lambda \leq   \alpha(t)\lambda_{\text{min}}^l\}$, where $\lambda_{\text{min}}^l$ is the smallest singular value of $\tilde X_{t-1}^l$ and $\alpha(t)$ is a positive value which balances the stability and plasticity. For the non-uniform constraint strength, we use a simple strategy that $\alpha(t)$ linearly decreases with task number $t$ observed so far.
  We perform experiments on the 10-Split CIFAR-100 and 20-Split CIFAR-100 5 runs, and 25-Spilt TinyImageNet 3 runs. 
  Note that the update of the null space in our method is only performed at the end of task, and no data of old data are stored during training.
  More implementation details, including the hyperparameters settings, can be found in Appendix.

\begin{table*}[t]
\centering
\caption{Results of ACC (\%) and BWT (\%) evaluated on the all tasks after finishing learning all tasks. [$\uparrow$] higher is better.
}
\resizebox{\columnwidth}{!}{
\begin{tabular}{@{}lcccccc@{}}
\toprule
\multirow{2}{*}{\textbf{Method}} &  \multicolumn{2}{c}{\quad \textbf{10-S-CIFAR-100}\quad}              & \multicolumn{2}{c}{\quad\textbf{20-S-CIFAR-100}\quad}          &          \multicolumn{2}{c}{\quad\textbf{25-S-TinyImagNet}\quad}              \\
            & \quad ACC [$\uparrow$]           & BWT [$\uparrow$]             &  \quad ACC [$\uparrow$]        & BWT [$\uparrow$]          & \quad ACC [$\uparrow$]    
                        
                        & BWT   [$\uparrow$]    \\
\midrule     
Vanilla             &     34.91    &    -60.96   &     30.48    &   -65.90      &    16.96       &  -66.06
\\
EWC \cite{kirkpatrick2017overcoming}              & 70.77          & -2.83       & 71.66          & -3.72         & 52.33          &  -6.71
\\
MAS\cite{aljundi2018memory}            &     66.93      &    -4.03    &    63.84       &   -6.29       &    47.96       & -7.04  \\
MUC-MAS\cite{liu2020more}          &      63.73     &   -3.38     &     67.22      &  -5.72        &     41.18      & -4.03 \\
SI \cite{zenke2017continual}         &      60.57     &   -5.17     &     59.76      &    -8.62      &      45.27     & -4.45 \\
CPR \cite{cha2020cpr}      &   74.56        &   -2.51    &      72.98  &   -2.32       &     58.01     & -2.45  \\
LwF\cite{li2017learning}         &        70.70   &   -6.27     &  74.38         &  -9.11        &    56.57       &  -11.19 \\
GD-WILD \cite{lee2019overcoming}          &     71.27      &    -18.24    &     77.16      &       -14.85   &    42.74      &  -34.58 \\
  InstAParam\cite{chen2020mitigating}     & 47.84          &     -11.92   &    51.04       & -4.92         &     34.64      & -10.05 \\
GEM \cite{lopez2017gradient}         &       49.48    &     2.77   &      68.89     &   -1.2       &     -      & - \\
A-GEM \cite{chaudhry2018efficient}           &    49.57     &    -1.13   &    61.91      &  -6.88      &    53.32        & -7.68 \\
  MEGA\cite{guo2020improved}       &     54.17      &  -2.19      &       64.98    &     -5.13     &         57.12  &  -5.90\\
OWM\cite{zeng2019continual}         &    68.89       &   -1.88     &      68.47      &    -3.37      &   49.98         & -3.64 \\
GPM \cite{saha2021gradient}              &   73.66        &    -2.20    &      75.20     &   -7.58      & 58.96          &  -6.96 \\
Adam-NSCL \cite{Wang_2021_CVPR}             & 75.03          & -2.98       & 75.59          & -3.66         & 59.10          &  -7.19 \\
\midrule 
AdNS (Ours)            &     \textbf{77.21}     &   -2.32    &     \textbf{77.33}      &   -3.25       & \textbf{59.77}          &   -4.58 \\
\bottomrule
\end{tabular}
    }
\label{LowRankNullSpace:table_acc_bwt}
\end{table*}
\subsection{Performance Comparison}
We show the comparison results of  the proposed method and baselines in Table \ref{LowRankNullSpace:table_acc_bwt}. The results execept Vanilla, CPR, and GPM are from \cite{Wang_2021_CVPR}. According to Table \ref{LowRankNullSpace:table_acc_bwt}, our method achieves the highest average accuracy (ACC) with comparable forgetting (BWT) on all benchmarks. 

Compared with regularization-based methods, the proposed method achieves over 5\% ACC higher with less forgetting than EWC and MAS on all benchmarks. Although MUC-MAS and SI obtained comparable forgetting on the 25-Spilt TinyImageNet, their ACCs are lower than 50\%, largely below AdNS's ACC (59.77\%). The forgetting of CPR on the 20-Split CIFAR-100 and 25-Spilt TinyImageNet are less, while its performance of ACC is worse than the proposed method on all benchmarks.
For the regularization-based methods using knowledge distillation,
the ACCs of LwF and GD-WILD are comparable to the proposed method on the 20-Split CIFAR-100 while their stability is very poor. 
For the architecture-based method, AdNS is significantly better than InstAParam, e.g., over 25\% ACC higher on three benchmarks.

Now we compare AdNS with algorithm-based methods.
On the 10-Split CIFAR-100, 
although the forgetting of GEM, A-GEM, MEGA, OWM, and GPM is slightly better than our method, their ACCs are largely below the proposed method.
It is also observed that our method can obtain better performance with less forgetting than Adam-NSCL.
 As shown in Table \ref{LowRankNullSpace:table_acc_bwt}, 
 the ACCs of the proposed method are 2.18\%, 1.74\%, and 0.66\% higher than Adam-NSCL on the 10-Split CIFAR-100, 20-Split CIFAR-100, and 25-Spilt TinyImageNet, respectively. 
As for forgetting, the BWTs of the proposed method are 0.66\%, 0.41\%, and 2.61\% better than Adam-NSCL on the three benchmarks, respectively.
 It is because AdNS considers the shared null space between null spaces 
 and also leverages knowledge distillation 
 to further mitigate the stability-plasticity dilemma.
\subsection{Ablation Studies and Analyses} 
 \begin{table*}[t]
\centering
\caption{Different methods to obtain the shared null space. ``Random" obtains the shared null space with the dimensions randomly from $\mathbf{U}_{\text{pre}}^l$ and $\mathbf{U}_{\text{cur}}^l$.
}
\resizebox{\columnwidth}{!}{
\begin{tabular}{@{}c|ccccccc@{}}
\toprule
 \multirow{2}{*}{\textbf{Method}}  &  \multicolumn{2}{c}{\quad \quad\textbf{10-S-CIFAR-100}\quad\quad }              & \multicolumn{2}{c}{\quad\quad \textbf{20-S-CIFAR-100}\quad \quad}          &          \multicolumn{2}{c}{\quad\quad \textbf{25-S-TinyImageNet}\quad\quad}              \\
 & \quad\quad \textbf{ACC} [$\uparrow$]          & \textbf{BWT} [$\uparrow$]          &  \quad\quad \textbf{ACC} [$\uparrow$]        & \textbf{BWT}   [$\uparrow$]          & \quad \quad \textbf{ACC } [$\uparrow$]      
                        
                        & \textbf{BWT} [$\uparrow$]       \\

\midrule 
Random        &     \quad   76.10     &   -4.13            &   \quad  75.70    &     -5.81      &  \quad  59.07  &  -6.78 \\
Low-Rank        &  \quad  \textbf{76.45}     &  \textbf{-2.87 }    & \quad    \textbf{76.31}      &   \textbf{-3.66 }    &  \quad   \textbf{59.26}       &  \textbf{-5.77} \\
\bottomrule
\end{tabular}
    }
\label{LowRankNullSpace:table:extract_share}
\end{table*}
\begin{figure}[t]
\centering
  \includegraphics[width=0.32\linewidth]{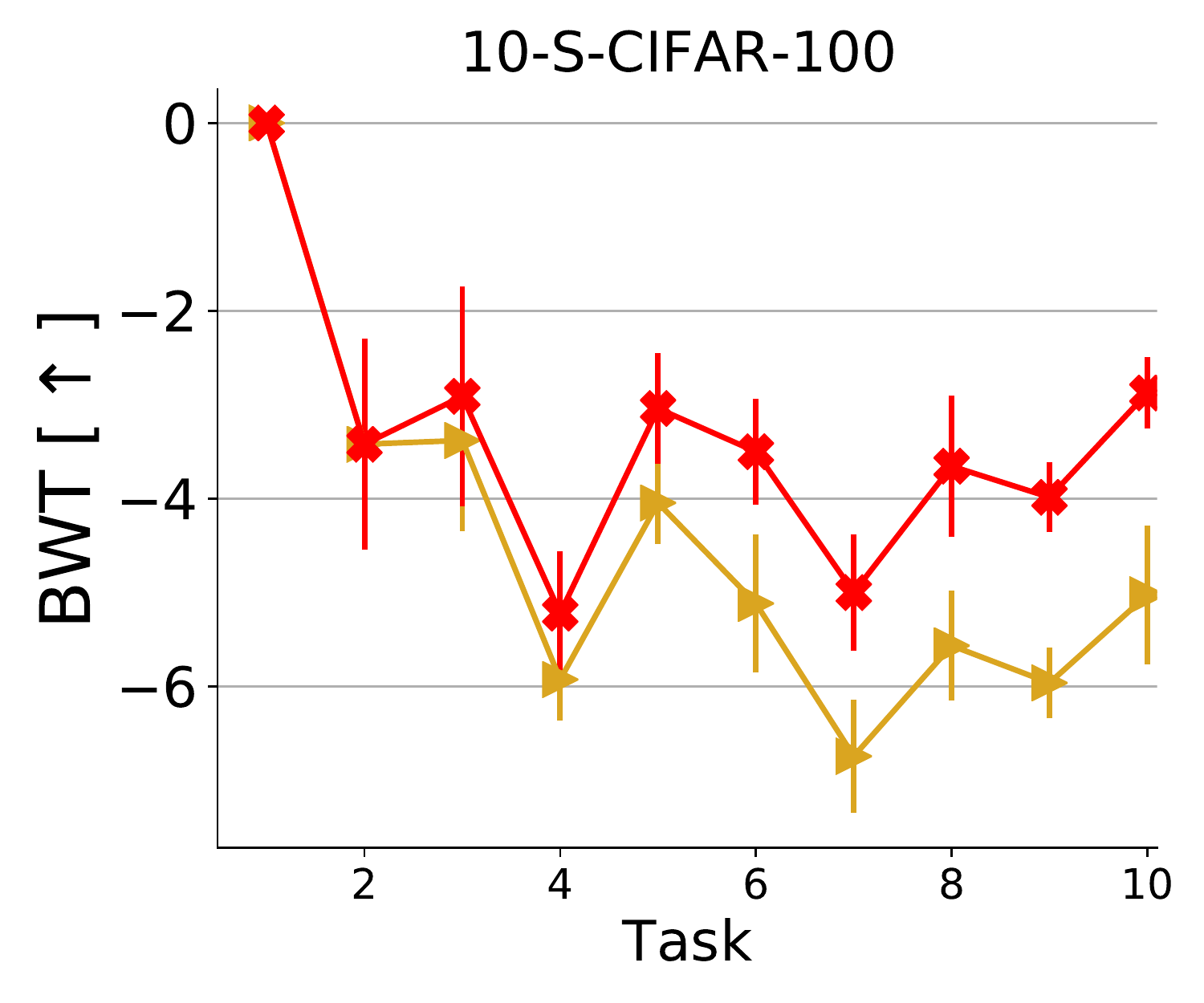}
   \includegraphics[width=0.32\linewidth]{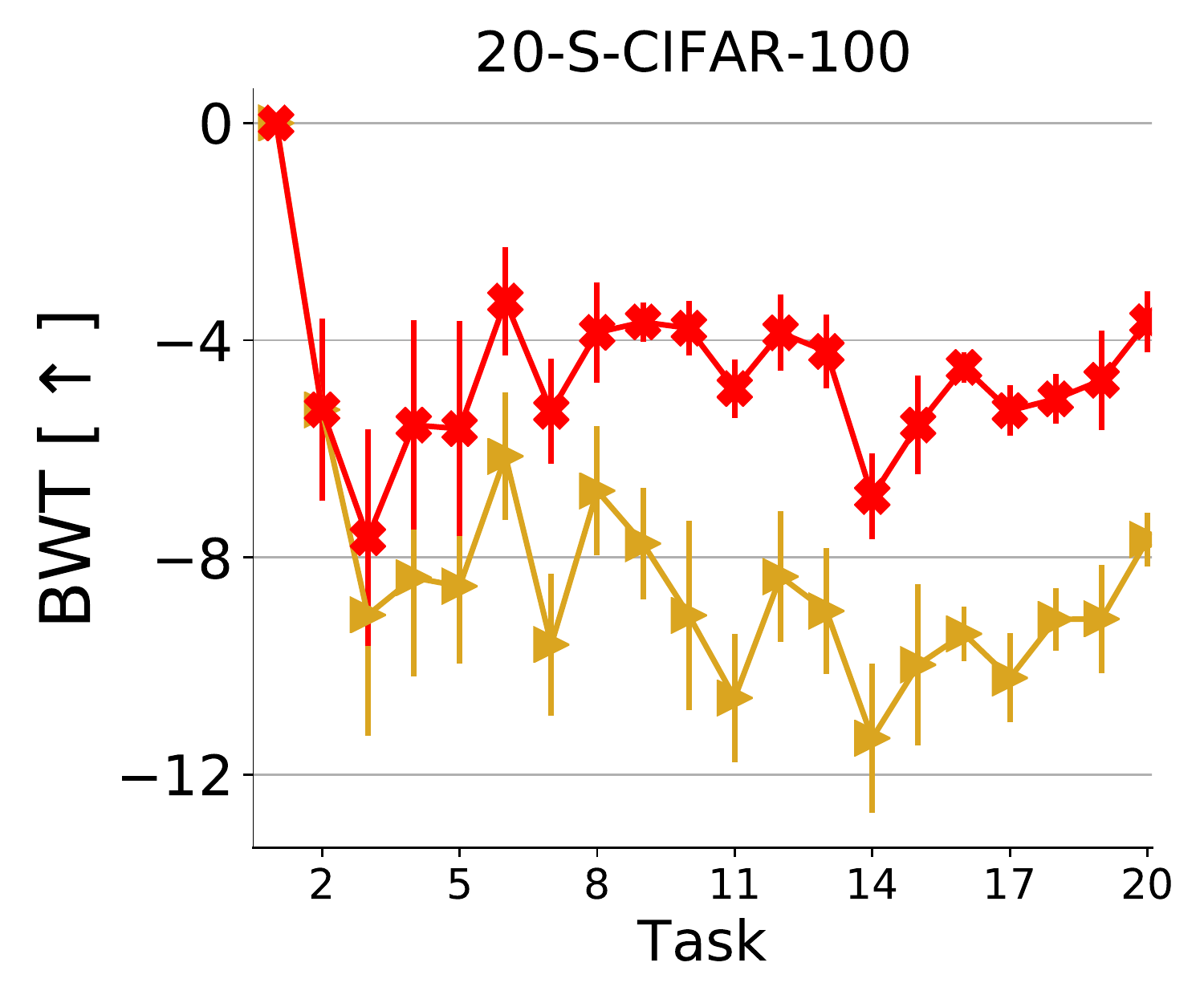}
   \includegraphics[width=0.32\linewidth]{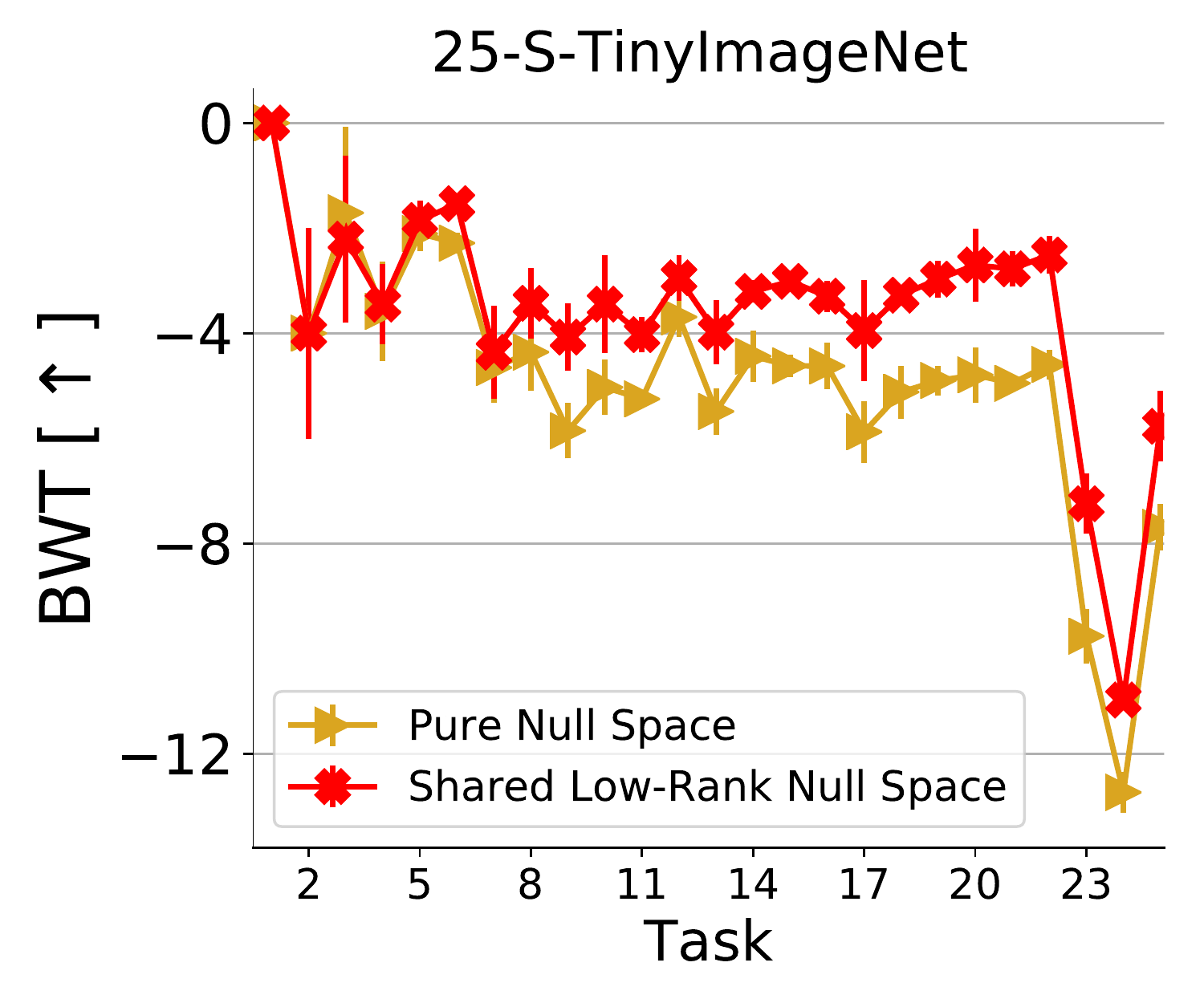}
   \caption{Comparison of forgetting between the pure null space and the shared low-rank null space. The results are the curves of BWT when the network has been trained on  each task ([$\uparrow$]  Higher BWT indicates less forgetting). 
   }
\label{LowRank:forgetting}
\end{figure}

\begin{table}[t]
\centering
\caption{Effect of each component. ``NS" denotes ``Null Space", ``LR" denotes ``Low-Rank", ``NCS" denotes ``Non-uniform Constraint Strength", and ``ID" denote Intra-task Distillation. 
}
\resizebox{\columnwidth}{!}{
\begin{tabular}{@{}ccccccccccc@{}}
\toprule
 \multicolumn{4}{c}{\textbf{Module}}  &  \multicolumn{2}{c}{ \textbf{\quad 10-S-CIFAR-100}\quad}              & \multicolumn{2}{c}{\quad \textbf{20-S-CIFAR-100}\quad}          &          \multicolumn{2}{c}{\quad\textbf{25-S-TinyImagnet}\quad}              \\
    \quad \textbf{NS} \quad &    \textbf{LR}     &  \textbf{NCS} &   \textbf{ID} \quad       & ACC [$\uparrow$]          & BWT  [$\uparrow$]         & \quad ACC  [$\uparrow$]        & BWT [$\uparrow$]          & \quad ACC   [$\uparrow$]    
                        & BWT    [$\uparrow$]    \\

\hline
   \Checkmark  &    &  &   &   76.11       &    -5.02    &      75.53     &  -7.67       &    59.20        &   -7.69 \\
      \Checkmark  & \Checkmark   &  &   & 76.45          &   -2.87     &  76.31         &  -3.66       & 59.26          &   -5.77 \\
            \Checkmark  &    & \Checkmark &   & 76.15          &   -4.98     &    75.66     & -7.36   &     59.42 &    -5.07          \\
 \Checkmark  &  \Checkmark   & \Checkmark &        &   76.45     &   -2.84     &  76.34         &    -3.55      & 59.27          &   -5.11 \\
    \Checkmark  &    &  &   \Checkmark   & 76.99         &   -4.36     &     77.04      &    -6.09     &     \textbf{60.00}      &   -6.98 \\
  \Checkmark  &        \Checkmark   & \Checkmark &   \Checkmark   &     \textbf{77.21}     &   \textbf{-2.32}    &      \textbf{77.33}    &     \textbf{-3.25 }    & 59.77          &   \textbf{-4.58} \\
\bottomrule
\end{tabular}
    }
\label{LowRankNullSpace:abla_modules}
\end{table}
\begin{figure}[t]
    \centering
    {
    \setlength\tabcolsep{0pt}
    \resizebox{1.0\textwidth}{!}{
    \begin{tabular}{cc}
        \includegraphics[width=0.3\textwidth]{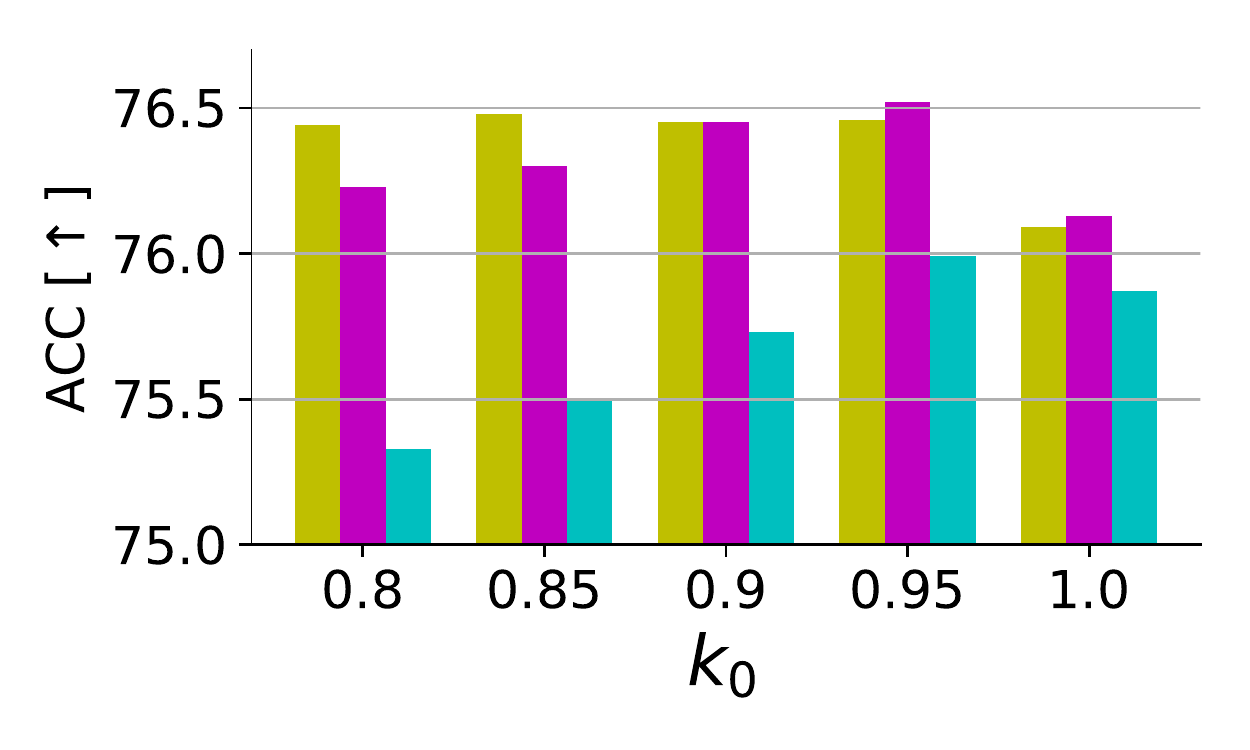} &
        \includegraphics[width=0.3\textwidth]{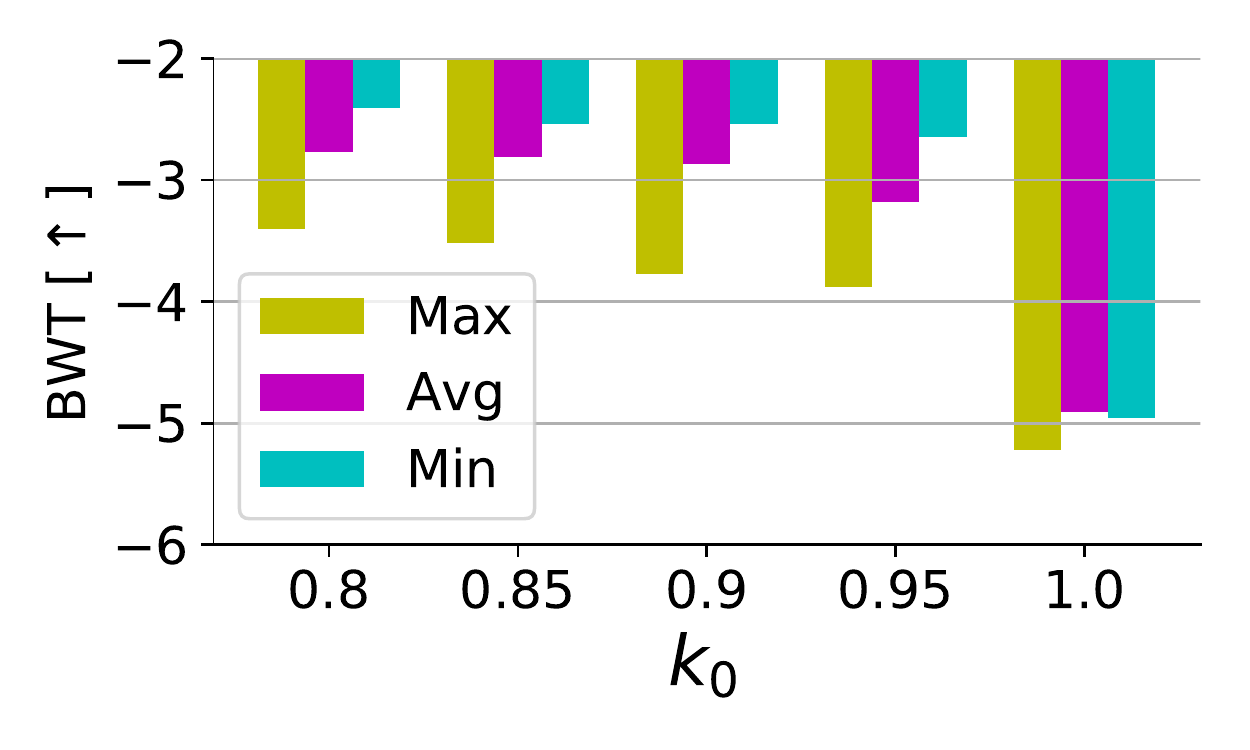} \\
    \end{tabular}
    }
    }
    \caption{The effect of $k$. The dataset is 10-Split CIFAR-100. 
    [$\uparrow$] higher is better.}
    \label{LowRankNullSpace:k_0}
\end{figure}

    \begin{figure*}[t]
\centering
   \includegraphics[width=0.85\linewidth]{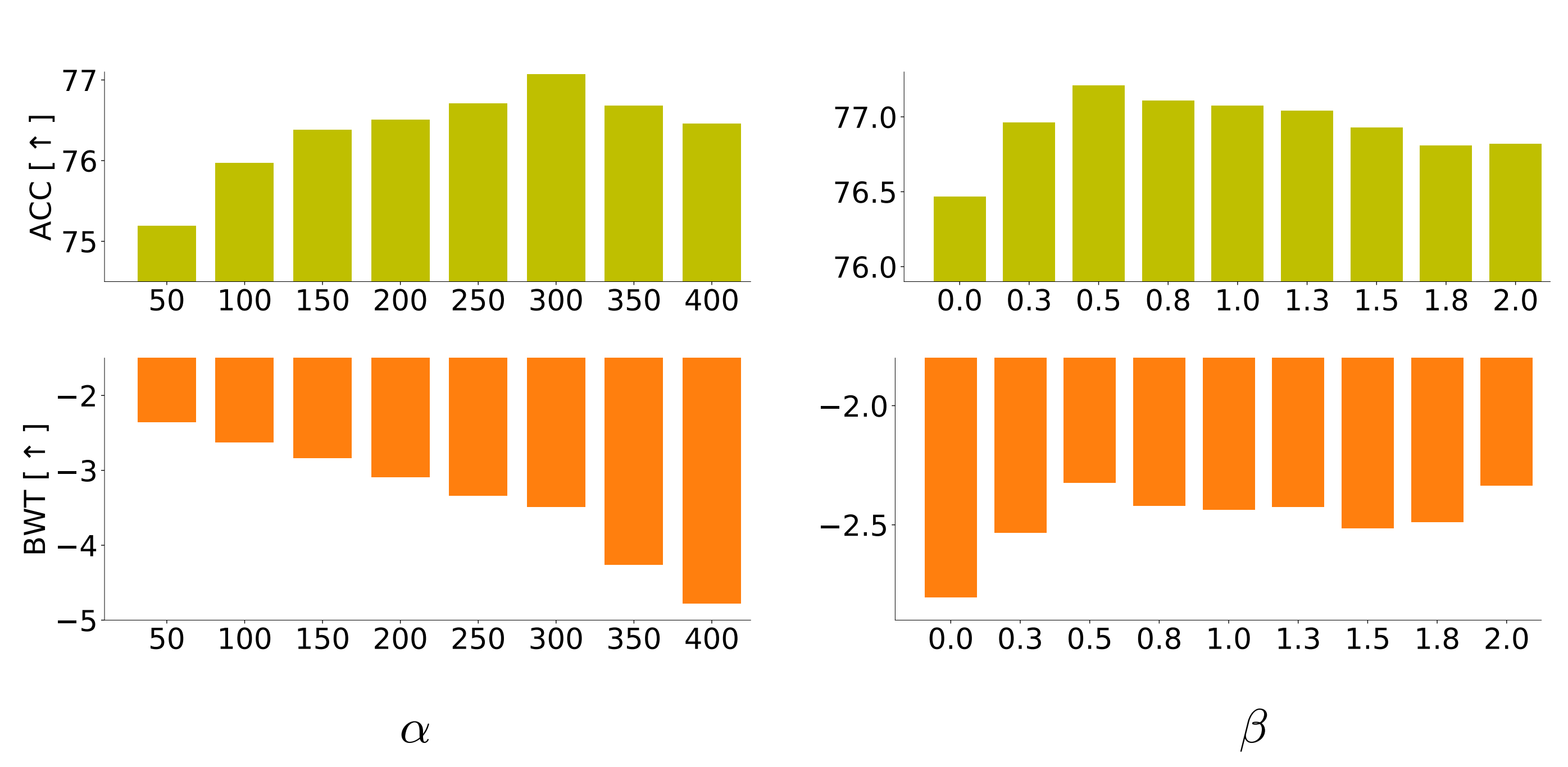}
   \caption{The effect of $\alpha$ and $\beta$. The dataset is 10-Split CIFAR-100. [$\uparrow$] higher is better. 
   }
\label{LowRank:figure_abla_a_beta}
\end{figure*}
  
\noindent \textbf{Effect of Low-Rank.} First, we compare the forgetting of using shared low-rank null space and pure null space, in which pure null space projects the gradient onto the current candidate null space instead of the shared null space. 
According to Fig. \ref{LowRank:forgetting}, with new tasks continually coming, the superiority of the shared low-rank null space on alleviating forgetting becomes more and more obvious than the pure null space, validating that our shared low-rank  null space can relieve the catastrophic forgetting due to null space approximation.
Next, we validate 
whether it is effective to use low-rank approximation to obtain the shared space. 
To realize this, we compare the performance of the shared low-rank null space with the method that obtains the shared null space by extracting  dimensions randomly from $\mathbf{U}_{\text{pre}}^l$ and $\mathbf{U}_{\text{cur}}^l$ \footnote{
We extract $k_l/2$ dimensions randomly from $\mathbf{U}_{\text{pre}}^l$ and another $k_l/2$ dimensions randomly  from $\mathbf{U}_{\text{cur}}^l$.
If the dimension of $\mathbf{U}_{\text{pre}}^l$ or  $\mathbf{U}_{\text{cur}}^l$ is smaller than $k_l/2$, to make up $k_l$ dimensions, we concatenate the whole matrix and the rest dimensions randomly extracted from another matrix.
}. For a fair comparison, intra-task distillation is excluded, and all the settings including the constraints are the same.  
According to Table \ref{LowRankNullSpace:table:extract_share}, using low-rank approximation 
can achieve higher ACC and better BWT on all benchmarks. Especially for BWT, ``Low-Rank" is at least 1\% better than the method of ``Random", indicating the effectiveness of low-rank approximation. 

\noindent \textbf{Effect of Each Component.}
We now validate the effect of each component to demonstrate that AdNS could achieve better stability-plasticity trade-off.
From Table \ref{LowRankNullSpace:abla_modules}, we can find that  (1) Both  ``Low-Rank" and ``Non-uniform Constraint Strength" decrease forgetting significantly with better performance. Especially, on the 20-Split CIFAR-100, the shared low-rank null space (``NS" + ``LR") is 4.01\% BWT better and 0.78\% ACC higher than the pure null space (``NS"). On the 25-Spilt TinyImageNet, adding non-uniform constraint strength (``NS" + ``NCS") increases BWT 2.62\%. 
(2) Intra-task Distillation improves ACC significantly with better BWT.  For example, on the 
20-Split CIFAR-100, adding intra-task distillation (``NS" + "ID") increases ACC 1.51 \% with less forgetting. (3) Combing all modules can achieve better stability and plasticity simultaneously. For example, on 20-Spilt CIFAR100, combing all modules is 1.8\% ACC  higher and 4.42\% BWT better than the pure null space (``NS").

\noindent \textbf{Effect of $k$.}  Now we explore the effect of $k$.  We use $k$ to denote $k_l$ because we apply the same operation for $k_l$ at each layer.
Assume that the dimensions of  $\mathbf{U}_{\text{pre}}^l$ and $\mathbf{U}_{\text{cur}}^l $ are $p$ and $q$, respectively. Then ``Max", ``Avg", and ``Min"\footnote{$\text{Max}(\cdot), \text{Avg}(\cdot)$, and $\text{Min}(\cdot)$ are functions that compute the maximum, average, and minimum values over inputs, respectively.}  means that $k = \text{Max}(p, q)\times k_0$, $k = \text{Avg}(p, q)\times k_0$, and $k = \text{Min}(p, q)\times k_0$, respectively, where $k_0$ is used to adjust the value of $k$. As shown in Fig. \ref{LowRankNullSpace:k_0}, for all strategies, with the decrease of $k_0$, forgetting continually decreases while ACC first increases and then decreases. It is because with the decrease of $k_0$, $k$ becomes smaller and the rank of null spaces is smaller.
According to Theorems \ref{Null_space:convergence_rate} and \ref{Null_space:forgetting}, it would result in better forgetting and worse plasticity, and thus the ACC first increases then decreases as a result of the stability-plasticity dilemma.
Considering both stability and plasticity, we choose the strategy of ``Avg'' for all experiments.

\noindent \textbf{Effect of $\alpha$ and $\beta$.} Finally, we explore the effect of $\alpha$ in constraint and $\beta$ in intra-task distillation, respectively. For simplification, we apply constant $\alpha$ for all tasks.
Larger $\alpha$ indicates looser constraint in  (\ref{LowRank:constraint}). 
As shown in Fig. \ref{LowRank:figure_abla_a_beta}, with the increase of $\alpha$, BWT becomes worse, and ACC first increases and then decreases as a result of the stability-plasticity dilemma.
The results agree with the theoretical analysis that larger null space (looser constraint) leads to better plasticity and worse stability and  vice-versa.
For intra-task distillation, as shown in Fig. \ref{LowRank:figure_abla_a_beta}, with proper $\beta$, it could mitigate the forgetting and achieve a good balance between stability and plasticity. Too large or too small $\beta$ will have a negative impact on the performance.

\section{Conclusion}
To relieve the stability-plasticity dilemma, we propose a new algorithm-based continual learning method, Advanced Null Space (AdNS).  Specifically, AdNS extracts the shared low-rank null space based on low-rank approximation and projects the gradient onto the null space to reduce forgetting. Moreover, we introduce 
non-uniform constraint strength to further alleviate the catastrophic forgetting, and present intra-task distillation to improve performance. Furthermore, we provide theoretical findings of the impact of null space on the stability and plasticity, respectively. Empirical results validate that the proposed algorithm could achieve a better stability-plasticity trade-off.

\subsubsection{Acknowledgements.} 
Ms Yajing Kong and Dr Liu Liu are supported by ARC FL-170100117 and DP-180103424.

\bibliographystyle{splncs04}
\bibliography{egbib}

\newpage
\appendix
\section*{ \centering Appendix}
\noindent The content of the appendix is arranged as follows:
\begin{itemize}
    \item [$\bullet$] {\bf  \ref{LowRank:Appendix:Exp} Experiments}
    \begin{itemize}
    \item[$\bullet$] {\bf \ref{LowRank:Appendix:Implementation details} Implementation Details} 
    
    In this part, we will introduce the low-rank approximation, details of $\alpha(t)$, $k_l$, and experimental setup. 
    \item[$\bullet$] {\bf \ref{LowRank:Appendix:Additional Experimental results} Additional Experimental Results}  
    
    The additional experimental results include the comparison of running time, and exploration of plasticity and stability.
    \end{itemize}
    \item[$\bullet$] {\bf \ref{lowrank:appendix:theo} Theoretical Analysis. } 
    In this section, we provide the proof of Theorems \ref{Null_space:convergence_rate} and \ref{Null_space:forgetting}, respectively.
\end{itemize}
\section{Experiments}
\label{LowRank:Appendix:Exp}
This section provides the implementation details in Appendix \ref{LowRank:Appendix:Implementation details} and the additional experimental results in Appendix \ref{LowRank:Appendix:Additional Experimental results}.
\subsection{Implementation Details}
\label{LowRank:Appendix:Implementation details}
\noindent \textbf{Low-Rank Approximation} \\
Now we describe how to obtain $\mathbf{U}^l$ by solving the problem of low-rank approximation for the concatenation matrix 
$\mathbf{\tilde U}^l = [\mathbf{U}_{\text{pre}}^l, \mathbf{U}_{\text{cur}}^l] \in \mathbb{R}^{d^l \times n_0}$:
\begin{align}
    \text{minimize}_{\mathbf{ \hat U}^l}  & \quad \| \mathbf{\tilde U}^l -  \mathbf{ \hat U}^l\|_F \nonumber \\
    \text{ s.t.} & \quad \text{Rank}(\mathbf{ \hat U}^l) \leq k_l, l \in \{1,...,L\},
\end{align}
where $d^l$ is the dimension of the feature at $l$-th layer and $n_0$ is the sum of the colomns of $\mathbf{U}_{\text{pre}}^l$ and $\mathbf{U}_{\text{cur}}^l$.
Let $\bar{\mathbf{U}},\bar{\Sigma}, \bar{\mathbf{V}}^{\top} = \text{SVD}(\mathbf{\tilde U}^l)$, where $\bar{\Sigma}$ is a diagonal matrix sorted by singular values. Then the low-rank approximation matrix $\mathbf{ \hat U}^l =  \mathbf{U}^l \Sigma^l (\mathbf{V}^l)^{\top}$,  where $\mathbf{U}^l \in \mathbb{R}^{d^l \times k_l}, \Sigma^l \in \mathbb{R}^{k_l \times k_l}, \mathbf{V}^l \in \mathbb{R}^{n_0 \times k_l}$; $\Sigma^l$ is a diagonal matrix sorted by the $k_l$ largest singular values, and $\mathbf{U}^l $ and  $\mathbf{V}^l$ are constructed by the singular vectors corresponding to the $k_l$ largest singular values in $\bar{\mathbf{U}}$ and $\bar{\mathbf{V}}$, respectively. Then $\mathbf{U}^l $ is the objective matrix whose columns span the shared low-rank null space.

\noindent \textbf{Details of $\alpha(t)$} \\
Instead of solving the constraint $\|X_{t-1}^l\Delta w^l\|_1 \leq \epsilon(t)$,
we use the function $\alpha(t)$ to replace $\epsilon(t)$ to balance the stability and plasticity.
We propose non-uniform constraint strength, which linearly decreases with the task number $t$, i.e., \textcolor{black}{$\alpha_{t} = \alpha_{\text{max}} - \frac{t-1}{T-1}(\alpha_{\text{max}} - \alpha_{\text{min}})$,} where $\alpha_{\text{max}}$ and $\alpha_{\text{min}}$ are the values of $\alpha(t)$ for the first and last task, respectively.

\noindent \textbf{Details of $k_l$} \\
When computing the shared low-rank null space, it is hard to set $k_l$ for all tasks and layers manually because we lack prior knowledge about the features at each layer. Therefore, we use a task-adaptive and layer-adaptive strategy to select $k_l$, i.e., we use the strategy of ``Avg" in the paper to select $k_l$ for all experiments. Specifically, assume that the dimensions of  $\mathbf{U}_{\text{pre}}^l$ and $\mathbf{U}_{\text{cur}}^l $ are $p$ and $q$, respectively. Then ``Avg" means that $k = \text{Avg}(p, q)\times k_0$, where $\text{Avg}(\cdot)$ computes the average value of $p$ and $q$, and $k_0$ is used to adjust the value of $k$. Note that we use the same operation for each layer. Therefore, we use $k$ to denote $k_l$ at each layer. Because the dimensions of the previous null space and the current candidate null space at each layer are determined by each task and the features at that layer, obtaining $k$ by such strategy is task adaptive and layer-adaptive. 

\noindent \textbf{Datasets} \\
The CIFAR100 dataset contains 100 classes, each of which has 500 training color images and 100 testing color images. 
TinyImageNet
is a 200 classes dataset which contains 100,000 training images and 10,000 validation images and consists of $64\times 64$ color images. 
\\
\noindent \textbf{Experimental Setup} \\
We use Pytorch\footnote{\url{https://pytorch.org/}} to implement the proposed algorithm and other experiments.
 The optimizer is Adam. Following \cite{Wang_2021_CVPR}, we use EWC \cite{kirkpatrick2017overcoming} to regularize the parameters of batch normalization layer and set the  regularization coefficient to 100. The learning rates for the first task are $2\times10^{-4}$, $1\times 10^{-4}$, and $1\times10^{-4}$ for 10-Split CIFAR-100, 20-Split CIFAR-100, and 25-Spilt TinyImageNet, respectively. Then the learning rate after the first task is set to $1\times10^{-4}$, $5\times 10^{-5}$, and $1\times10^{-4}$, respectively. After that, we delay the learning rate at epoch 30 and 60 by multiplying with 0.5. The total epoch is 80 for all benchmarks. The batch size of 10-Split CIFAR-100, 20-Split CIFAR-100, and 25-Spilt TinyImageNet are 32, 16, and 16, respectively. We set $k_0=0.9$ for all benchmarks. The $\alpha_{\text{max}}$ are 160, 180, and 20, and $\alpha_{\text{min}}$ are 150, 150, and 5 for 10-Split CIFAR-100, 20-Split CIFAR-100, and 25-Spilt TinyImageNet, respectively.

\subsection{Additional Experimental Results}
\label{LowRank:Appendix:Additional Experimental results}
\subsubsection{Comparison of Running Time}
In this part, we compare the running time of the proposed method with Adam-NSCL \cite{Wang_2021_CVPR}, the most related baseline, to validate that the time consumption of the proposed method is comparable. The device is a single Nvidia Tesla V100 (16GB) GPU. As shown in Table \ref{LowRankNullSpace:RunningTime}, the  running time of AdNS is comparable to Adam-NSCL on all benchmarks, validating that the time consumption of low-rank approximation is moderate. 
\begin{table}[t]
\centering
\caption{Comparison of running time (s). The device is a single Nvidia Tesla V100 (16GB) GPU. The results of 10-Split CIFAR-100 and 20-Split CIFAR-100 are the average running time over five repetitions, and the results of 25-Split TinyImageNet are over three repetitions.
}
\resizebox{\columnwidth}{!}{
\begin{tabular}{@{}c|cccc@{}}
\toprule
 Method  &  \textbf{10-S-CIFAR-100}              &\textbf{20-S-CIFAR-100}          &          \textbf{25-S-TinyImageNet}              \\

\midrule 
Adam-NSCL \cite{Wang_2021_CVPR}        &      5167 $\pm$ 5     &   23175 $\pm$ 416           &    32676 $\pm$ 1237    \\
AdNS (Ours)        &       5623 $\pm$ 13     &   26286 $\pm$ 501            &    33462 $\pm$ 172
 \\
\bottomrule
\end{tabular}
    }
\label{LowRankNullSpace:RunningTime}
\end{table}
\begin{figure*}[t]
    \centering
    {
    \setlength\tabcolsep{0pt}
    \resizebox{1.0\textwidth}{!}{
    \begin{tabular}{ccc}
        \includegraphics[width=1.0\textwidth]{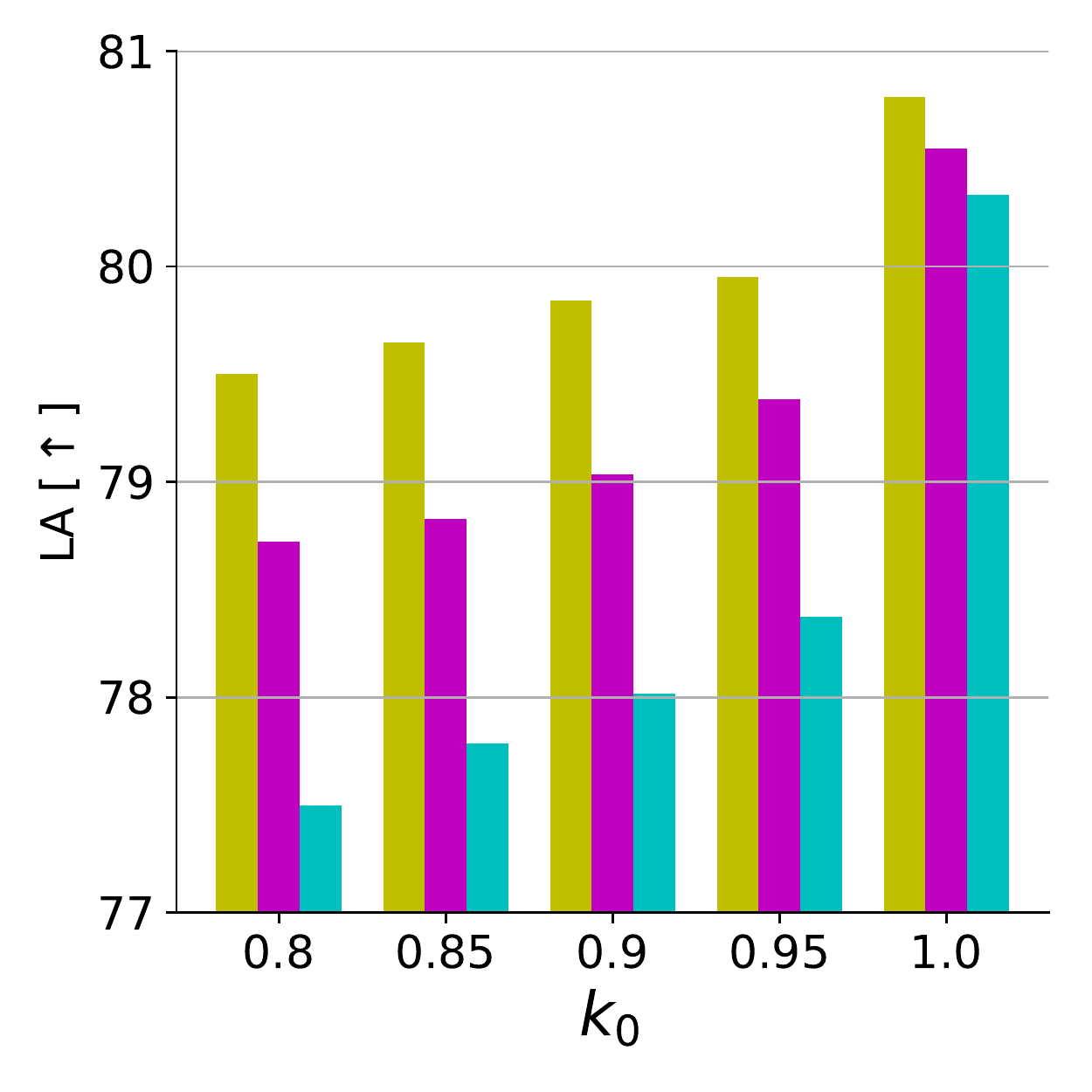}&
        \includegraphics[width=1.0\textwidth]{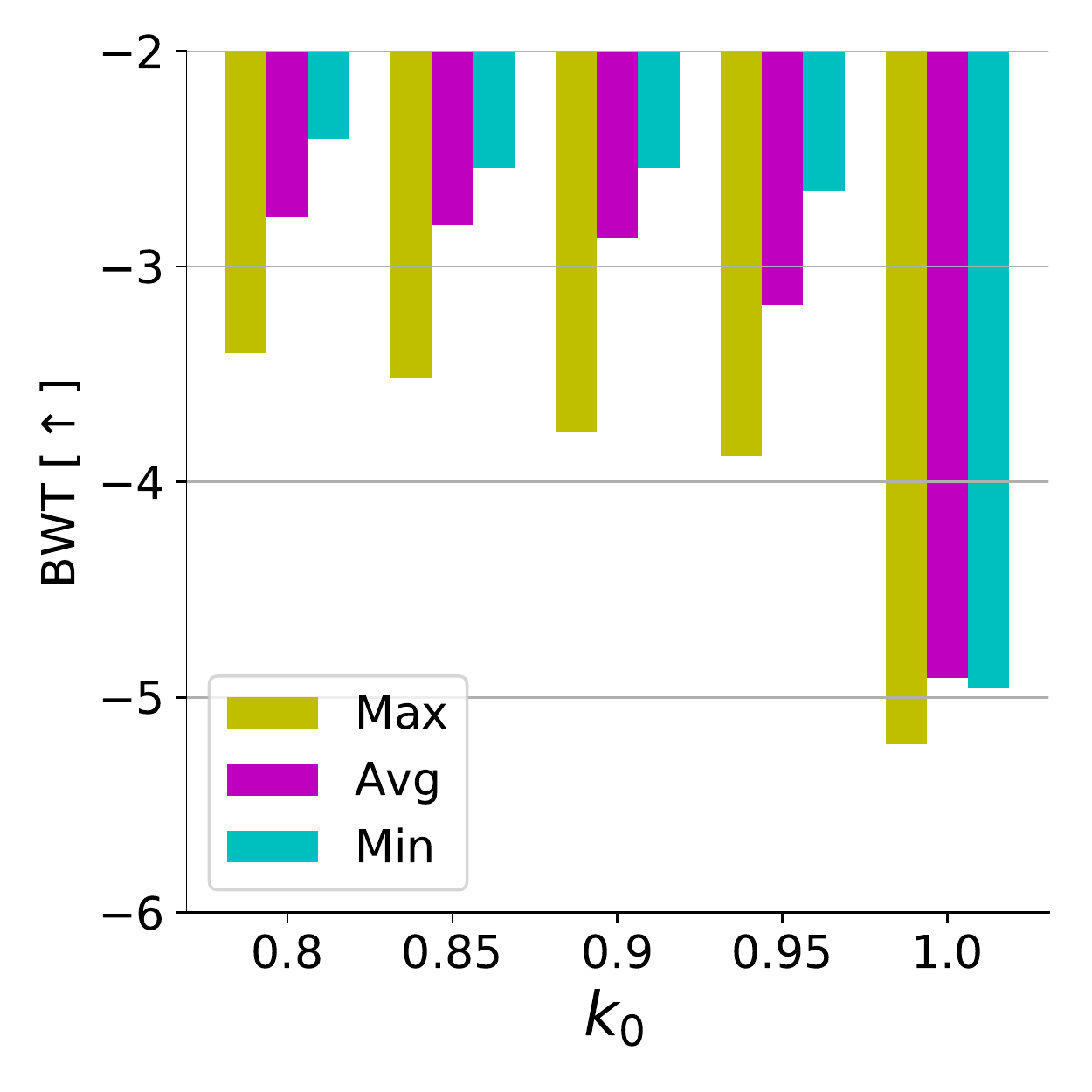} &
        \includegraphics[width=1.0\textwidth]{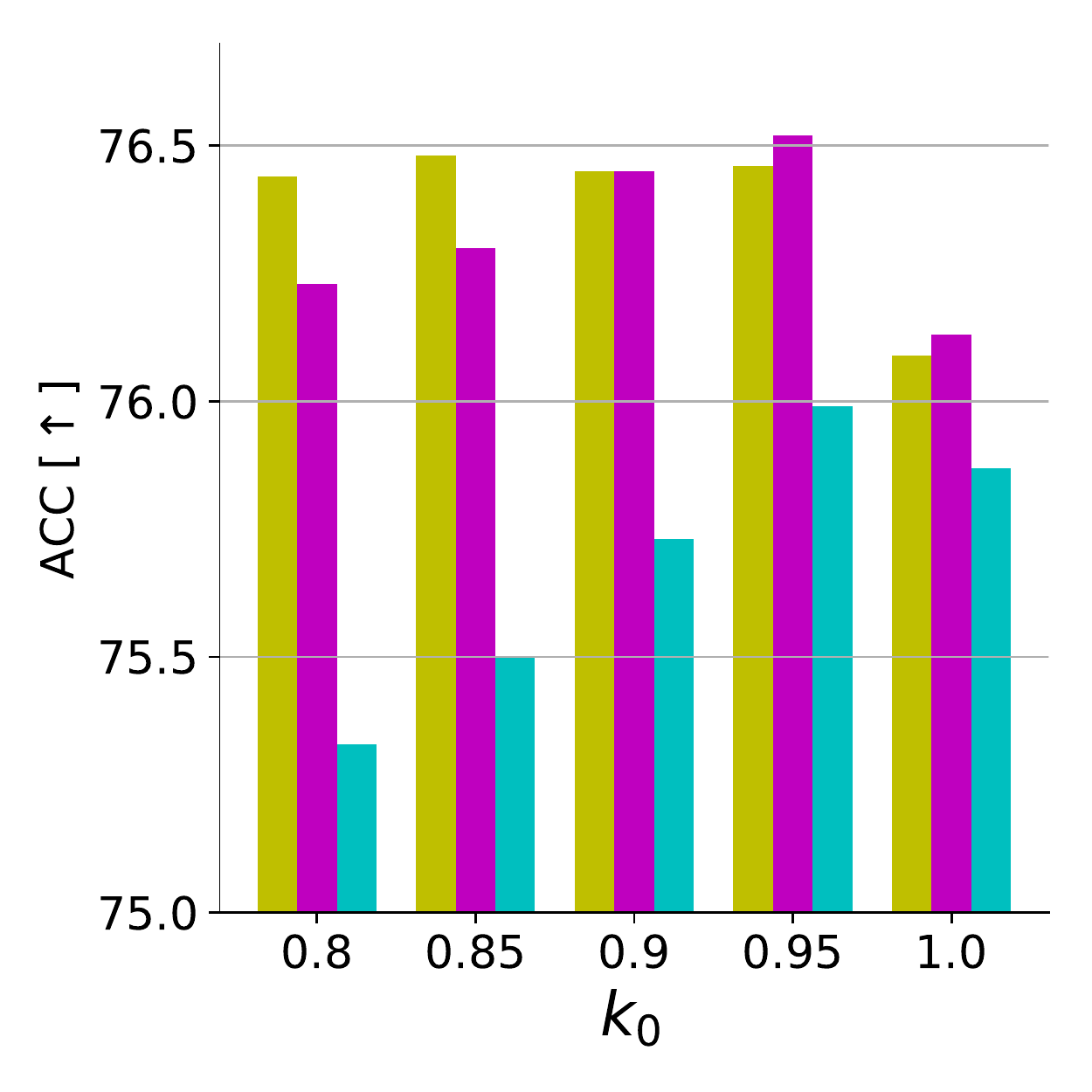} \\
    \end{tabular}
    }
    }
    \caption{The effect of $k$. Higher BWT indicates less forgetting (high stability) and higher LA indicates higher plasticity.
    The dataset is 10-Split CIFAR-100.
    [$\uparrow$] higher is better. }
    \label{LowRankNullSpace:k_0_pls_acc_bwt}
\end{figure*}

\subsubsection{Plasticity and Stability}
\noindent \textbf{The effect of $k$.}
Like previous works
\cite{riemer2018learning}, 
we use Learning Accuracy (LA), which is the accuracy of the model on a task
right after it finishes training the task, to measure the plasticity. Higher LA indicates better plasticity.
In formal, LA is defined by
 \begin{equation}
     \text{LA} = \frac{1}{T}\sum_{i=1}^{T}A_{i, i},
 \end{equation}
where $T$ is the total number of tasks, and $A_{i, i}$ is the accuracy of task $\mathcal{T}_i$ after training on the task $\mathcal{T}_i$ sequentially. The larger the LA, the better the plasticity of the model. 
As shown in Fig. \ref{LowRankNullSpace:k_0_pls_acc_bwt}, for all strategies, with the increase of $k_0$, the plasticity becomes better  while the forgetting becomes worse. 
The ACC first increases and then decreases due to the stability-plasticity trade-off. 
The results agree with the theoretical findings.
With the increase of $k_0$, $k$ becomes larger and the rank of the shared low-rank null space is larger.
According to Theorems \ref{Null_space:convergence_rate} and \ref{Null_space:forgetting}, it would result in better plasticity and worse forgetting. Finally, the performance of ACC will first increase and then decrease due to the stability-plasticity dilemma.

\noindent \textbf{The effect of $\alpha$.} Larger $\alpha$ indicates looser constraint in  (\ref{LowRank:constraint}), i.e., $\| X_{t-1}^l\Delta w^l \|_1 \leq \epsilon(t)$. 
As shown in Fig. \ref{LowRankNullSpace:a_pls_acc_bwt}, with the increase of $\alpha$, LA becomes higher while BWT becomes worse.
ACC first increases and then decreases as a result of the stability-plasticity dilemma.

\begin{figure*}[t]
    \centering
    {
    \setlength\tabcolsep{0pt}
    \resizebox{1.0\textwidth}{!}{
    \begin{tabular}{ccc}
        \includegraphics[width=1.0\textwidth]{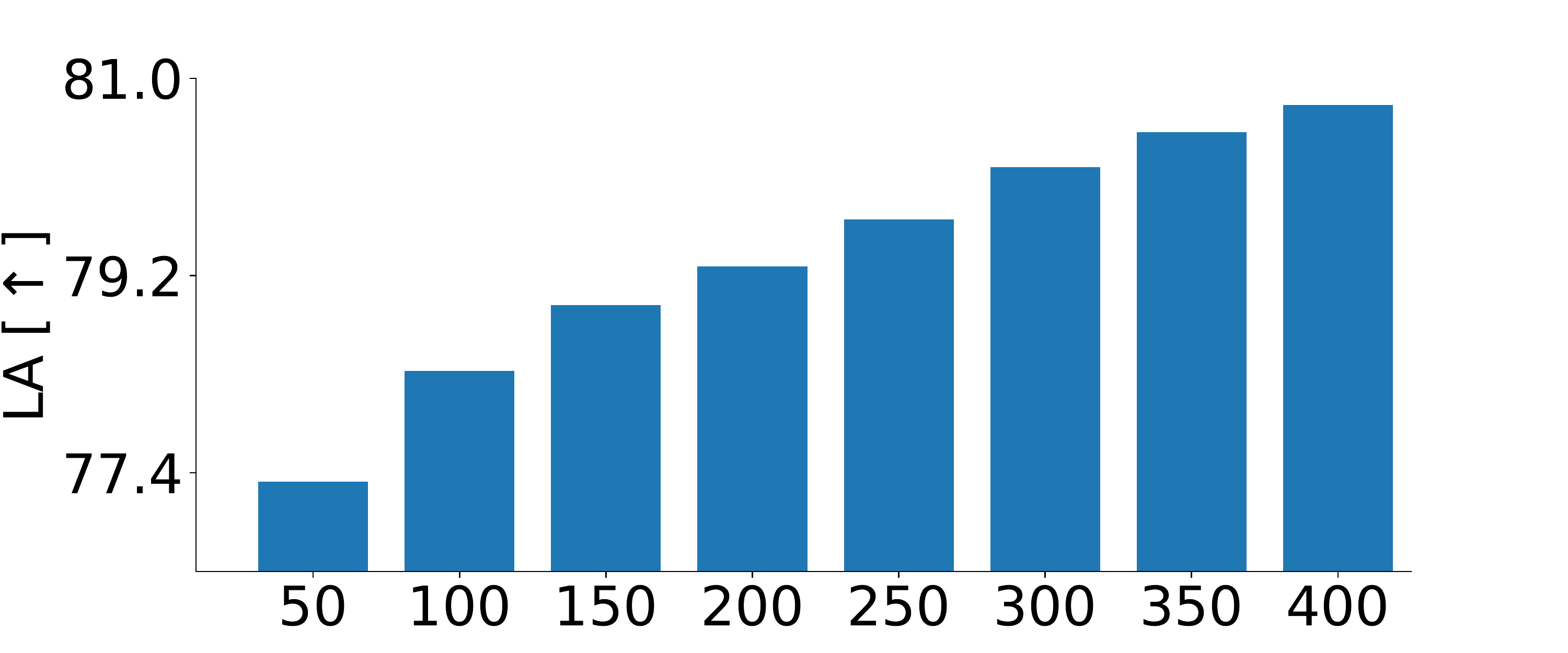}&
        \includegraphics[width=1.0\textwidth]{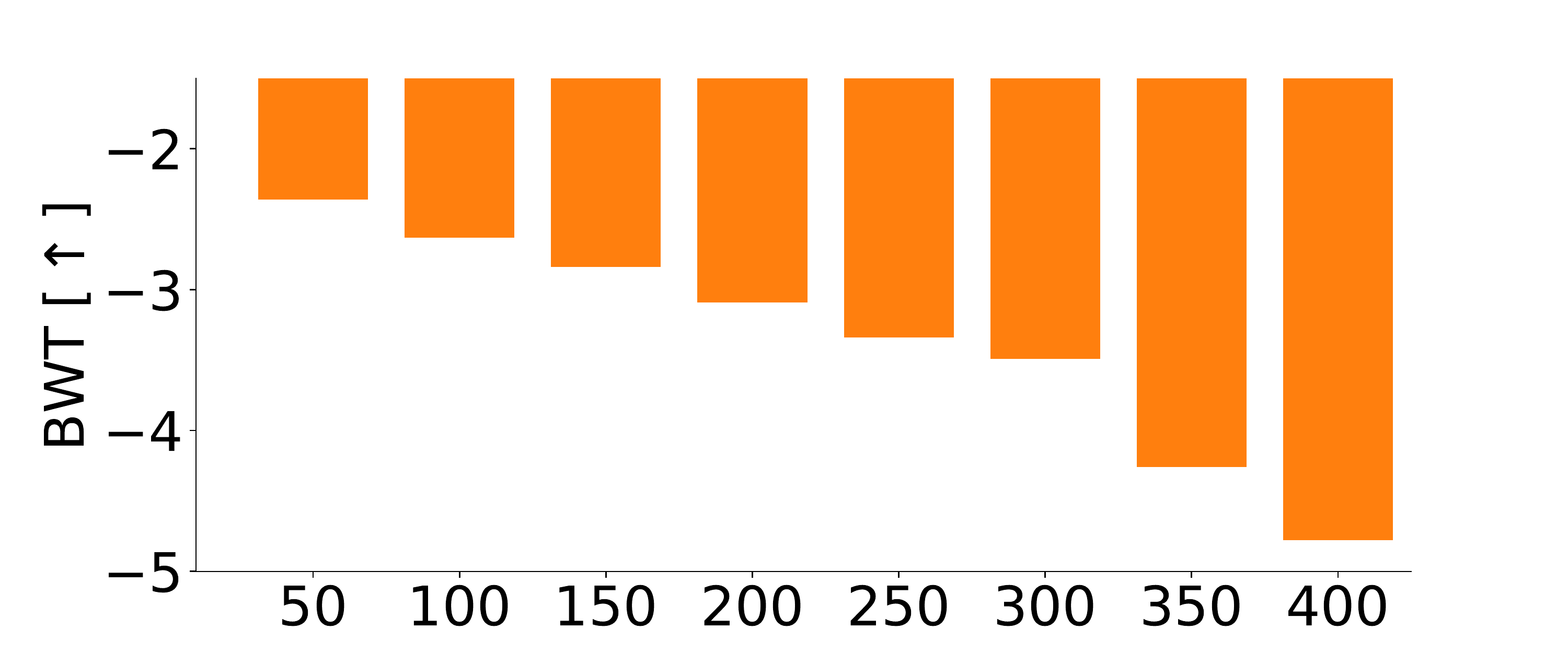} &
        \includegraphics[width=1.0\textwidth]{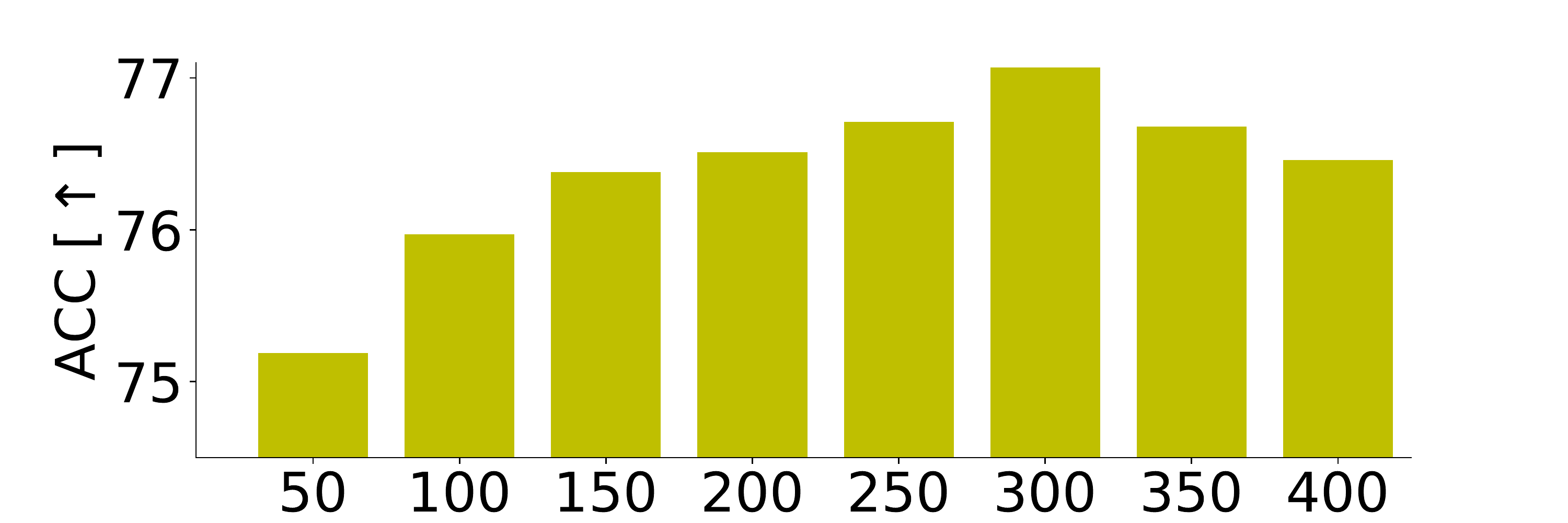} \\
    \end{tabular}
    }
    }
    \caption{The effect of $\alpha$. The x-axis is the value of $\alpha$, and the y-axes are different metrics. Higher BWT indicates less forgetting (high stability) and higher LA indicates higher plasticity.
    The dataset is 10-Split CIFAR-100.
    [$\uparrow$] higher is better. }
    \label{LowRankNullSpace:a_pls_acc_bwt}
\end{figure*}

\noindent \textbf{The plasticity of different values of $\beta$.}
Now we validate that intra-task distillation is beneficial to improve the performance of the current task. Fig. \ref{LowRank:beta_pls} shows that with proper $\beta$, the LA can be improved, validating the effectiveness of intra-task distillation on improving the performance of the current task.
 \begin{figure}[t]
\centering
   \includegraphics[width=0.5\linewidth]{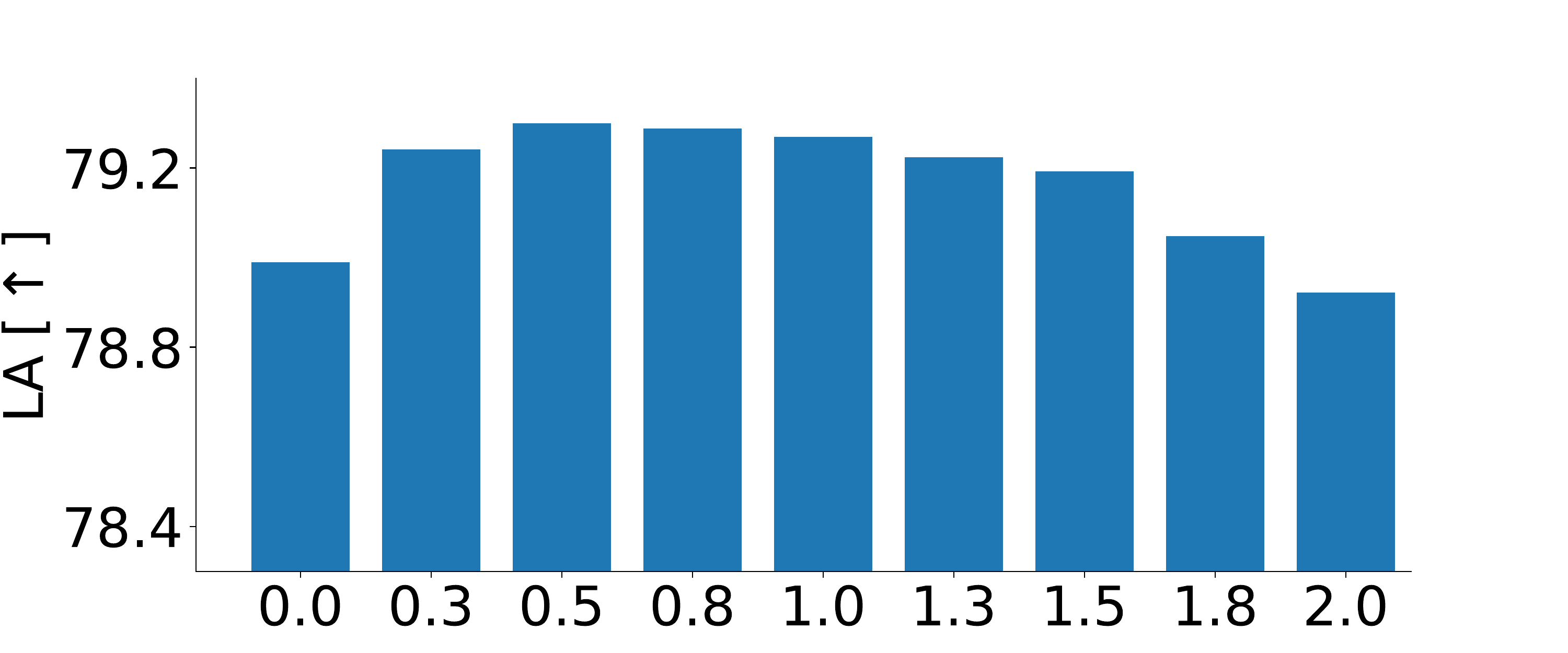}
   \caption{The plasticity of different values of $\beta$. The x-axis is the value of $\beta$, and the y-axis is Learning Accuracy (LA) . Higher LA indicates higher plasticity.
   }
\label{LowRank:beta_pls}
\end{figure}

\section{Theoretical Analysis}
\label{lowrank:appendix:theo}
In this part, we present the proof of Theorems \ref{Null_space:convergence_rate} and \ref{Null_space:forgetting}, respectively. 
\begin{lemma}
\label{LowRank:Lemma:smooth:upperbound-all}
    Assume that $f(x)$ is L-smooth, let learning rate be $\eta\leq \frac{1}{L}$ and the update be $x_{t+1}=x_t-\eta \nabla h(x_t)$, we have 
    \begin{align*}
    \frac{\eta}{2} \|\nabla f(x_t)\|^2_2
    \leq & 
    f(x_t)
    -\mathbb{E}[f(x_{t+1})] +\frac{L\eta^2}{2} \mathbb{E}\|\nabla h(x_t)-\mathbb{E}[\nabla h(x_t)]\|^2_2\\
    & +\frac{1}{2}\eta {\|\nabla f(x_t)-\mathbb{E}[\nabla h(x_t)]\|^2_2}.
    \end{align*}
\end{lemma}
\begin{proof}
Based on the definition of the smoothness of $f(x)$ and the update of $x_{t+1}=x_t-\eta \nabla h(x_t)$, we obtain,
 \begin{align*}
    f(x_{t+1})=&f(x_t-\eta \nabla h(x_t))
    \\
    \le & f(x_t)+\langle \nabla f(x_t),x_{t+1}-x_t \rangle +\frac{L}{2}\|x_{t+1}-x_t\|^2_2
    \\
    =& f(x_t)- \langle \nabla f(x_t), 
   {\color{black}\eta \nabla h (x_t) }
    \rangle +\frac{L}{2}\eta^2\|
    {\color{black}\nabla h(x_t)}
    \|^2_2.
\end{align*}
Taking expectation on both sides, we
\allowdisplaybreaks[4]
{\small
\begin{align*}
    \mathbb{E}[f(x_{t+1})]
    \le & f(x_t)- \mathbb{E}\langle \nabla f(x_t), 
   {\color{black}\eta \nabla h (x_t) }
    \rangle +\frac{L}{2}\eta^2\mathbb{E}\|
    {\color{black}\nabla h(x_t)}
    \|^2_2
    \\
    =&f(x_t)- \eta \langle \nabla f(x_t), {\color{black}\mathbb{E}[\nabla h(x_t)]}
    \rangle +\frac{L}{2}\eta^2\mathbb{E}\|\nabla {h}(x_t)
    {\color{black} 
    -\mathbb{E}[\nabla h(x_t)]
    +\mathbb{E}[\nabla h(x_t)]
    }
    \|^2_2
    \\
    =&f(x_t)
    -{\color{black}
    \frac{1}{2}\eta\left(
     \|\nabla f(x_t)\|^2_2
     +\|\mathbb{E}[\nabla h(x_t)]\|^2_2
    -\|\nabla f(x_t)-\mathbb{E}[\nabla h(x_t)]\|^2_2
    \right)
    }
    \\
    +\frac{L\eta^2}{2}&\left(
    {\color{black}
    \|\mathbb{E}[\nabla h(x_t)]\|^2_2
    \!+\!\mathbb{E}\|\nabla h(x_t)\!-\!\mathbb{E}[\nabla h(x_t)]\|^2_2
    \!+\!2\mathbb{E}\langle \mathbb{E}[\nabla h(x_t)],\nabla h(x_t)\!-\!\mathbb{E}[\nabla h(x_t)]\rangle
    }
    \right)
        \\
    =&f(x_t)
    -\frac{1}{2}\eta\left(
     \|\nabla f(x_t)\|^2_2
     +\|\mathbb{E}[\nabla h(x_t)]\|^2_2
    -\|\nabla f(x_t)-\mathbb{E}[\nabla h(x_t)]\|^2_2
    \right)
    \\
    &+\frac{L}{2}\eta^2\left(
    {\color{black}
    \|\mathbb{E}[\nabla h(x_t)]\|^2_2
    +\mathbb{E}\|\nabla h(x_t)-\mathbb{E}[\nabla h(x_t)]\|^2_2
    }
    \right)
    \\ 
    =&f(x_t)
    -\frac{1}{2}\eta \|\nabla f(x_t)\|^2_2
    -\left(\frac{1}{2}\eta-\frac{1}{2}L\eta^2\right)\|\mathbb{E}[\nabla h(x_t)]\|^2_2
    \\
    &+\frac{1}{2}L\eta^2 \mathbb{E}\|\nabla h(x_t)-\mathbb{E}[\nabla h(x_t)]\|^2_2
    \\
    &+\frac{1}{2}\eta {\|\nabla f(x_t)-\mathbb{E}[\nabla h(x_t)]\|^2_2},
\end{align*}
}
where the second equation is based on the fact that $a^2+b^2-(a-b)^2=2ab$.
By arranging, we have
\begin{align*}
    \frac{1}{2}\eta \|\nabla f(x_t)\|^2_2
    +\left(\frac{1}{2}\eta-\frac{1}{2}L\eta^2\right)\|\mathbb{E}[\nabla h(x_t)]\|^2_2
    \leq &
    f(x_t)
    -\mathbb{E}[f(x_{t+1})] 
    \\
    &+\frac{1}{2}L\eta^2 \mathbb{E}\|\nabla h(x_t)-\mathbb{E}[\nabla h(x_t)]\|^2_2\\
    &+\frac{1}{2}\eta {\|\nabla f(x_t)-\mathbb{E}[\nabla h(x_t)]\|^2_2}.
\end{align*}
If $\eta\leq \frac{1}{L}$, we have 
\begin{align*}
    \frac{1}{2}\eta \|\nabla f(x_t)\|^2_2
    \leq&
    f(x_t)
    -\mathbb{E}[f(x_{t+1})]
    +\frac{1}{2}L\eta^2 \mathbb{E}\|\nabla h(x_t)-\mathbb{E}[\nabla h(x_t)]\|^2_2\\
    &+\frac{1}{2}\eta {\|\nabla f(x_t)-\mathbb{E}[\nabla h(x_t)]\|^2_2}.
\end{align*}
\end{proof}

\begin{theorem2} \textbf{\ref{Null_space:convergence_rate}} (Plasticity) Suppose Assumptions \ref{LowRank:Assumption:sigmoid-distance}, \ref{LowRank:Assumption:sameT}, and  \ref{LowRank:Assumption:sigma} hold. Let $\mathbf{w}_{t, s}$ be the parameters on task $\mathcal{T}_t$ at the $s$-th step. Let the range of space of $\mathbf{U}^l$ be the null space of previous tasks for $l$-th layer,
then the loss of the current task $\mathcal{T}_{t}$ after training on task $\mathcal{T}_t$ is upper bound by
{\small
\begin{align*}
    \hat L_t(\mathbf{w}_{t,S}) 
    \!\leq\! \hat L_t(\mathbf{w}_{t, 0})
    &\!+\!\frac{\eta}{2}\sum_{s=0}^{S-1}\sum_{l = 1}^{L}{\|(I\!-\!
    \mathbf{U}^l(\mathbf{U}^l)^{\top})g_{t, s}^{l}\|_2^2} \nonumber \!-\!\frac{1}{2}\eta\sum_{s=0}^{S-1}\|\nabla \hat L_t(\mathbf{w}_{t,s})\|_2^2 
    \!+\! \frac{1}{2}S L_f \eta^2\sigma^2,
\end{align*}
}
where $g_{t, s}^{l}$ is $l$-th layer gradient of  $ \hat L_t(\mathbf{w}_{t, s})$.
\end{theorem2}

\begin{proof}
When training on the task $\mathcal{T}_t$, based on Lemma \ref{LowRank:Lemma:smooth:upperbound-all}, if $\eta \leq \frac{1}{L_f}$ we have 
{\small
\begin{align}
    \hat L_t(\mathbf{w}_{t, s+1}) \leq & \hat L_t(\mathbf{w}_{t, s}) + \frac{1}{2}L_f\eta^2\sigma^2 - \frac{1}{2}\eta\|\nabla \hat L_t(\mathbf{w}_{t, s})\|^2_2 \\
    & + \frac{1}{2}\eta \|\nabla \hat L_t(\mathbf{w}_{t, s})-   \begin{bmatrix}
	 (\mathbf{U}^1(\mathbf{U}^1)^{\top}(g_{t, s}^{1})^{\top}  &  ... &  (\mathbf{U}^L(\mathbf{U}^L)^{\top}(g_{t, s}^{L})^{\top}
    \end{bmatrix}^{\top}\|^2_2\\
    = & \hat L_t(\mathbf{w}_{t, s}) + \frac{1}{2}L_f\eta^2\sigma^2 - \frac{1}{2}\eta\|\nabla \hat L_t(\mathbf{w}_{t, s})\|^2_2  + \frac{1}{2}\eta \sum_{l = 1}^{L}{\|(I-  \mathbf{U}^l(\mathbf{U}^l)^{\top})g_{t, s}^{l}\|^2_2}. 
\end{align}
}
Summing from $s=0$ to $s=S - 1$, we obtain
{\small
    \begin{align}
    \hat L_t(\mathbf{w}_{t,S}) 
    \!\leq\! \hat L_t(\mathbf{w}_{t, 0})
    &\!+\!\frac{\eta}{2}\sum_{s=0}^{S-1}\sum_{l = 1}^{L}{\|(I\!-\!
    \mathbf{U}^l(\mathbf{U}^l)^{\top})g_{t, s}^{l}\|_2^2} \nonumber \!-\!\frac{1}{2}\eta\sum_{s=0}^{S-1}\|\nabla \hat L_t(\mathbf{w}_{t,s})\|_2^2 
    \!+\! \frac{1}{2}S L_f \eta^2\sigma^2.
    \end{align}
}
\end{proof}
\begin{theorem2} \textbf{\ref{Null_space:forgetting}} (Stability) Suppose Assumptions \ref{LowRank:Assumption:sigmoid-distance} and \ref{LowRank:Assumption:sameT} hold. Let $\mathbf{w}_{t, s}$ be the parameters on task $\mathcal{T}_t$ at the $s$-th. Let $\hat L_{1:t-1}$ be the sum of empirical loss function of previous $t-1$ tasks and $g_{1:t-1, s}^{l}$ is its gradient of $l$-th layer at $\mathbf{w}_{t, s}$. 
Let
$g_{t, s}^{l}$ be the gradient of the current task at $\mathbf{w}_{t, s}$ of $l$-th layer. 
Let the range of space of $\mathbf{U}^l$ be the null space of previous tasks for $l$-th layer, then the forgetting of previous $t-1$ tasks generated by the training on the task $\mathcal{T}_t$ is upper bound by
\begin{align*}
    \hat L_{1:t-1}(\mathbf{w}_{t, S})   -  \hat L_{1:t-1}(\mathbf{w}_{t, 0}) 
    \leq  &
    \eta \sum_{s=0}^{S-1}\sum_{l=1}^{L} \|\mathbf{U}^l(\mathbf{U}^l)^{\top}\|_2 \|g_{t, s}^l\|_2 \| g_{1:t-1, s}^{l}\|_2 
    \\
    &+ \frac{L_f}{2}\eta^2\sum_{s=0}^{S-1}\sum_{l=1}^{L}\|\mathbf{U}^l(\mathbf{U}^l)^{\top}\|_2^2\|g_{t, s}^l\|_2^2.
\end{align*}
\end{theorem2}
\begin{proof}
\textcolor{black}{Because the distributions between tasks are different, we could assume that $\left < \Delta \mathbf{w}_{t, s}, g_{1:t-1, s} \right > \leq 0$. }
Based on the smoothness of $\hat L_{1:t-1}$, we obtain

\begin{align*}
    \hat L_{1:t-1}(\mathbf{w}_{t, s+1}) 
    & \leq 
    \hat L_{1:t-1}(\mathbf{w}_{t, s}) - \eta \left <\Delta \mathbf{w}_{t, s}, \mathbf{g}_{1:t-1, s} \right >+ \frac{L_f}{2}\sum_{l=1}^{L}\|\eta \Delta \mathbf{w}_{t, s}\|^2_2 \\
    & = 
    \hat L_{1:t-1}(\mathbf{w}_{t, s}) - \eta \left <\Delta \mathbf{w}_{t, s}, \mathbf{g}_{1:t-1, s} \right >+ \frac{L_f}{2}\eta^2\sum_{l=1}^{L}\|\mathbf{U}^l(\mathbf{U}^l)^{\top}g_{t, s}^l\|^2_2 \\
    & \leq 
    \hat L_{1:t-1}(\mathbf{w}_{t, s}) - \eta \left <\Delta \mathbf{w}_{t, s}, \mathbf{g}_{1:t-1, s} \right >+ \frac{L_f}{2}\eta^2\sum_{l=1}^{L}\|\mathbf{U}^l(\mathbf{U}^l)^{\top}\|^2_2\|g_{t, s}^l\|^2_2,
\end{align*}
where $\mathbf{w}_{t, s+1} = \mathbf{w}_{t, s} - \eta \Delta \mathbf{w}_{t, s}$, $\mathbf{g}_{1:t-1, s}$ is the gradient of $\hat L_{1:t-1}$, and
$g_{1:t-1, s}^{l}$ is the gradient of $\hat L_{1:t-1}$ at $l$-th layer.

Because of
\begin{align*}
   \left< \Delta \mathbf{w}_{t, s}, \mathbf{g}_{1:t-1, s} \right> & = 
     \textcolor{black}{\sum_{l=1}^{L} \left<\mathbf{U}^l(\mathbf{U}^l)^{\top}g_{t, s}^l,  g_{1:t-1, s}^{l}  \right> \leq \sum_{l=1}^{L} \|\mathbf{U}^l(\mathbf{U}^l)^{\top}\|_2 \|g_{t, s}^l\|_2 \| g_{1:t-1, s}^{l}\|_2,}
\end{align*}
we have
\begin{align*}
    \hat L_{1:t-1}(\mathbf{w}_{t, s+1}) \leq & \hat L_{1:t-1}(\mathbf{w}_{t, s}) + \eta \sum_{l=1}^{L} \|\mathbf{U}^l(\mathbf{U}^l)^{\top}\|_2 \|g_{t, s}^l\|_2 \| g_{1:t-1, s}^{l}\|_2 
    \\
    &+ \frac{L_f}{2}\eta^2\sum_{l=1}^{L}\|\mathbf{U}^l(\mathbf{U}^l)^{\top}\|^2_2\|g_{t, s}^l\|^2_2.
\end{align*}
\textcolor{black}{Summing from $s=0$ to $s=S - 1$, we obtain}
\begin{align*}
 \hat L_{1:t-1}(\mathbf{w}_{t, S})   -  \hat L_{1:t-1}(\mathbf{w}_{t, 0}) 
    \leq  &
    \eta \sum_{s=0}^{S-1}\sum_{l=1}^{L} \|\mathbf{U}^l(\mathbf{U}^l)^{\top}\|_2 \|g_{t, s}^l\|_2 \| g_{1:t-1, s}^{l}\|_2 
    \\
    &+ \frac{L_f}{2}\eta^2\sum_{s=0}^{S-1}\sum_{l=1}^{L}\|\mathbf{U}^l(\mathbf{U}^l)^{\top}\|_2^2\|g_{t, s}^l\|_2^2.
\end{align*}
\end{proof}

\end{document}